\documentclass[11pt]{article}
\usepackage[bookmarks,colorlinks,hyperfootnotes=false]{hyperref}
\usepackage{xcolor}
\definecolor{darkblue}{rgb}{0,0.22,0.66}
\definecolor{darkcyan}{RGB}{0, 139, 139}
\definecolor{darkgray}{HTML}{666666}
\hypersetup{colorlinks=true,
citecolor=darkcyan,
linkcolor=darkblue,
urlcolor=darkgray
}

\usepackage[a4paper]{geometry}		
\usepackage{amsthm}
\usepackage{amsmath,scalerel}
\usepackage[numbers,sort]{natbib} 

\usepackage{amsmath,amsfonts,amssymb}
\usepackage[scr=rsfs]{mathalpha} 
\usepackage{mathtools}
\usepackage{dsfont} 
\usepackage{bm}

\usepackage{xurl,xcolor
}

\usepackage{enumitem}

\setlist[enumerate]{
  topsep=3pt,
  itemsep=3pt,
  parsep=0pt,
  partopsep=0pt,
  leftmargin=*,   %
  align=right,
  labelsep=0.600em,
  itemindent=0.00em , %
  labelindent=0.6em,
}

\setlist[itemize]{
  topsep=3pt,
  itemsep=3pt,
  parsep=0pt,
  partopsep=0pt,
  leftmargin=*,   %
  align=right,
  labelsep=0.600em,
  itemindent=0.00em , %
  labelindent=0.6em,
}

\usepackage{tabularx,multirow,array,booktabs}
\usepackage{colortbl}
\usepackage{graphicx}
\graphicspath{{figures/}{figuresOld/}}%
\usepackage{float}
\usepackage{subcaption} 

\usepackage{caption}
\usepackage{amsthm}
\theoremstyle{plain}
\newtheorem{theorem}{Theorem}[section]%

\newtheorem*{example*}{Example}
\newtheorem{corollary}[theorem]{Corollary} %
\newtheorem{lemma}[theorem]{Lemma}
\newtheorem*{lemma*}{Lemma}

\newtheorem{proposition}[theorem]{Proposition}
\theoremstyle{definition}

\theoremstyle{remark}

\newtheorem*{remark*}{Remark}
\usepackage{etoolbox,xparse,dsfont,bm,amssymb}
\newcommand{\myMFabc}[4]{\expandafter#1\csname#3#4\endcsname{{#2{#4}}}}
\newcommand{\myMFcmd}[4]{\expandafter#1\csname#3#4\endcsname{{#2{\csname#4\endcsname}}}}

\newcommand{\MFabc}[3][\newcommand]{
    \def\doOld##1##2{\forcsvlist{\myMFabc{#1}{##1}{##2}}{#3}}
    \providecommand{\do}{do}
    \RenewDocumentCommand \do { >{\SplitList{,}} m } { \doOld##1 }
    \docsvlist{#2}
}
\newcommand{\MFcmd}[3][\newcommand]{
    \def\doOld##1##2{\forcsvlist{\myMFcmd{#1}{##1}{##2}}{#3}}
    \providecommand{\do}{do}
    \RenewDocumentCommand \do { >{\SplitList{,}} m } { \doOld##1 }
    \docsvlist{#2}
}

\let\one\bbone
\MFabc{ {\mathfrak,frak}, }{T,u,U,o,O,E,m}
\MFabc{ {\mathbb, }, }{A,N,Z,R,C,Q}

\newcommand{\hatbm}[1]{\widehat{\bm{#1}}}
\newcommand{\tildebm}[1]{\widetilde{\bm{#1}}}
\newcommand{\bmcal}[1]{\bm{\mathcal{#1}}}
\newcommand{\caltilde}[1]{\mathcal{\widetilde{#1}}}
\newcommand{\calhat}[1]{\mathcal{\widehat{#1}}} 
\newcommand{\scrtilde}[1]{\widetilde{\mathscr{#1}\mspace{1mu}\mspace{-1mu}}}
\newcommand{\scrhat}[1]{\mathscr{\widehat{#1}\mspace{1mu}\mspace{-1mu}}}
\newcommand{\bmcalhat}[1]{\bm{\mathcal{\widehat{#1}}}}
\newcommand{\bmcaltilde}[1]{\bm{\mathcal{\widetilde{#1}}}}
\MFabc{ {\mathcal,cal}, {\mathscr,scr}, {\mathbb,bb}, {\bmcal,bmcal}, {\bmcal,calbm}, {\caltilde,caltilde}, {\caltilde,tildecal}, {\calhat,calhat}, {\calhat,hatcal}, {\scrtilde,scrtilde}, {\scrtilde,tildescr}, {\scrhat,scrhat}, {\scrhat,hatscr}, {\bmcalhat,bmcalhat}, {\bmcalhat,bmhatcal}, 
{\bmcalhat,calbmhat}, {\bmcalhat,calhatbm}, {\bmcalhat,hatbmcal}, {\bmcalhat,hatcalbm}, {\bmcaltilde,bmcaltilde}, {\bmcaltilde,bmtildecal}, {\bmcaltilde,calbmtilde}, {\bmcaltilde,caltildebm}, {\bmcaltilde,tildebmcal}, {\bmcaltilde,tildecalbm}}{A,B,C,D,E,F,G,H,I,J,K,L,M,N,O,P,Q,R,S,T,U,V,W,X,Y,Z}

\MFabc{{\bm,bm}, {\mathsf,sf}, {\widehat,hat}, {\widetilde,tilde}, {\hatbm,hatbm}, {\hatbm,bmhat}, {\tildebm,tildebm}, {\tildebm,bmtilde}}{a,b,c,d,e,f,g,h,i,j,k,l,m,n,o,p,q,r,s,t,u,v,w,x,y,z,A,B,C,D,E,F,G,H,I,J,K,L,M,N,O,P,Q,R,S,T,U,V,W,X,Y,Z}

\MFcmd{{\bm,bm}, {\widehat,hat}, {\widetilde,tilde}, {\tildebm, tildebm}, {\tildebm, bmtilde}, {\hatbm, hatbm}, {\hatbm, bmhat}}{alpha,beta,gamma,delta,epsilon,zeta,eta,theta,iota,kappa,lambda,mu,nu,xi,omicron,pi,rho,sigma,tau,upsilon,phi,chi,psi,omega,Alpha,Beta,Gamma,Delta,Epsilon,Zeta,Eta,Theta,Iota,Kappa,Lambda,Mu,Nu,Xi,Omicron,Pi,Rho,Sigma,Tau,Upsilon,Phi,Chi,Psi,Omega,varrho,varphi,vartheta,varepsilon,varsigma,ell,eps}

\newcommand{\actdef}[1]{\expandafter\def\csname#1\endcsname{{\ensuremath{\mathtt{#1}}}}}
\forcsvlist{\actdef}{ReLU, LReLU, LeakyReLU, ELU, GELU, SiLU, Softplus, dGELU, dSiLU, dSoftplus, Tanh, Sigmoid, Arctan, Softsign, SRS, dSRS, Swish, dSwish, Mish, dMish, SELU, CELU, dSELU, Sin,SinLU, SinTU, PSinTU, sine, cosine, Sine, Cosine, EUAF}

\newlength{\myLength}

\newcommand{\mystep}[2]{\par \vspace{0.25cm}\noindent\textbf{\hspace{8pt}Step }$#1\colon$ #2 \vspace{0.18cm} \par }

\newenvironment{keywords}{\par \noindent\textbf{Keywords}:}{\par}
\newenvironment{MSCcodes}{\par\par \noindent\textbf{MSC codes}:}{\par}

\usepackage[left,mathlines]{lineno}
\usepackage{refcount}

\definecolor{mylinenumbercolor}{HTML}{BEBEBE}

\linenumbers

\makeatletter
\newcommand*\patchAmsMathEnvironmentForLineno[1]{%
	\expandafter\let\csname old#1\expandafter\endcsname\csname #1\endcsname
	\expandafter\let\csname oldend#1\expandafter\endcsname\csname end#1\endcsname
	\renewenvironment{#1}%
	{\linenomath\csname old#1\endcsname}%
	{\csname oldend#1\endcsname\endlinenomath}}%
\newcommand*\patchBothAmsMathEnvironmentsForLineno[1]{%
	\patchAmsMathEnvironmentForLineno{#1}%
	\patchAmsMathEnvironmentForLineno{#1*}}%
\patchBothAmsMathEnvironmentsForLineno{equation}%
\patchBothAmsMathEnvironmentsForLineno{align}%
\patchBothAmsMathEnvironmentsForLineno{flalign}%
\patchBothAmsMathEnvironmentsForLineno{alignat}%
\patchBothAmsMathEnvironmentsForLineno{gather}%
\patchBothAmsMathEnvironmentsForLineno{multline}%
\makeatother

\let\epsilon\varepsilon

\let\tn\textnormal

\let\cdots\customcdots
\let\myforall\forall
\def\forall{{\myforall\, }}
\let\myexists\exists
\def\exists{{\myexists\, }}

\long\def\blue#1{{\color{blue}#1}}

\definecolor{mygray}{RGB}{230,230,230}
\definecolor{myorange}{HTML}{ff7f0e}

\let\cite\citep
\usepackage{doi}

\makeatletter
\long\def\@makefntext#1{\@setpar{\@@par\@tempdima \hsize 
		\advance\@tempdima-15pt\parshape \@ne 15pt \@tempdima}\par
	\parindent 2em\noindent \hbox to \z@{\hss{\textsuperscript{\@thefnmark}} \hfil}#1}
\newlength\aftertitskip     \newlength\beforetitskip
\newlength\interauthorskip  \newlength\aftermaketitskip
\def\maketitle{\par
	\begingroup
	\def\thefootnote{\color{black}\fnsymbol{footnote}}
	\def\@makefnmark{\hbox to 0pt{$^{\@thefnmark}$\hss}}
	\@maketitle \@thanks
	\endgroup
	\setcounter{footnote}{0}
	\let\maketitle\relax \let\@maketitle\relax
	\gdef\@thanks{}\gdef\@author{}\gdef\@title{}\let\thanks\relax}

\def\@startauthor{\noindent \normalsize\bf}
\def\@endauthor{}
\def\@starteditor{\noindent \small {\bf Editor:~}}
\def\@endeditor{\normalsize}
\def\@maketitle{\vbox{\hsize\textwidth
		\linewidth\hsize \vskip \beforetitskip
		{\begin{center} \Large\bf \@title \par \end{center}} \vskip \aftertitskip
		{\def\and{\unskip\enspace{\rm and}\enspace}%
			\def\addr{\small\it}%
            \def\email{\hfill\small\ttfamily}%
			\def\name{\normalsize\bf}%
			\def\AND{\@endauthor\rm\hss \vskip \interauthorskip \@startauthor}
			\@startauthor \@author \@endauthor}
		\vskip \aftermaketitskip
}}
\makeatother
\usepackage{scalerel,amssymb}

\let\ocirc\ocomp

\def\aff{{\mathsf{Aff}}}

\newenvironment{colorenv}[1][blue]
{%
  \begingroup
  \color{#1}%
  \ignorespaces
}
{%
  \endgroup
  \ignorespacesafterend
}

 \def\ACHACleanCopy{}

\ifdefined\ACHACleanCopy
  \colorlet{blue}{black}
  
\else
  \definecolor{blue}{rgb}{0,0,1}
\fi

\renewenvironment{keywords}{\par\noindent\textbf{Key words.}\enspace}{\par}
\renewenvironment{MSCcodes}{\par\noindent\textbf{MSC codes.}\enspace}{\par}
\usepackage[expansion=false]{microtype}
\usepackage{placeins}
\usepackage{needspace}
\pretocmd{\mystep}{\Needspace{4\baselineskip}}{}{}

\title{Layer-wise Geometric Approximation Rates for Deep Networks}

\author{\name Shijun Zhang\thanks{Corresponding author
}
\email \href{mailto:shijun.zhang@polyu.edu.hk}{shijun.zhang@polyu.edu.hk}\\
\addr Department of Applied Mathematics\\
Hong Kong Polytechnic University
\AND
\name Zuowei Shen
\email 
\href{mailto:matzuows@nus.edu.sg}{matzuows@nus.edu.sg}\\
\addr Department of Mathematics\\
National University of Singapore
\AND \name Yuesheng Xu
\email \href{mailto:y1xu@odu.edu}{y1xu@odu.edu}
\\ 
\addr Department of Mathematics and Statistics\\ Old Dominion University
 }

\begin{document}
\maketitle

\begin{abstract}
Depth is widely viewed as a central contributor to the success of deep neural networks, yet standard approximation results typically give error bounds only for the final output of a network. Inspired by classical multiscale expansions and multigrade deep learning, we give depth a scale-dependent interpretation by constructing a single network whose intermediate readouts approximate the target at successively finer scales. For any $f\in L^p([0,1]^d)$ with $1\le  p <\infty$ and any integer $N\ge2$, we construct a single network of width $3dN+2d+4$ with the same activation pattern in every hidden layer, using sine in the first two neurons and $\mathtt{ReLU}$ in the rest. For every $\ell\ge0$, the $L^p$ approximation error of the readout $\Phi_\ell$ at layer $\ell+4$ is bounded by $(2d+1)\omega_{f,p}(N^{-\ell})$, where $\omega_{f,p}$ denotes the $L^p$ translation modulus of $f$. Under the H\"older condition $\omega_{f,p}(t)\le Ct^\alpha$ with $C>0$ and $0< \alpha \le 1$, this estimate yields the geometric bound $(2d+1)C N^{-\alpha\ell}$. More generally, if $\omega_{f,p}(t)$ has a known upper bound that tends to zero as $t\downarrow0$, we obtain a stopping rule for adaptive depth selection. We can stop at the first depth where the error bound meets the desired tolerance and later append layers to meet a tighter tolerance while keeping all earlier maps and readouts fixed.
\end{abstract}

\vspace{5pt}

\section{Introduction}\label{sec:intro}

In scientific computing, the behavior of a function at one scale can differ substantially from its behavior at finer scales.
A coarse approximation may capture the dominant behavior, while higher accuracy may require resolving variation on finer scales.
Multiscale methods use the coarse approximation as a starting point and add corrections as finer detail is needed.
For example, a Fourier expansion can be extended by adding higher modes, and a wavelet expansion by adding details at the next resolution. In both cases, the coefficients already computed remain fixed \cite{devore_1998,Ingrid,cohen2020optimal,RON1997408,WaveletTour2009}.
Deep neural networks also build representations in stages through repeated composition. Each block can reuse and transform features computed by earlier blocks. This gives depth its expressive potential, but it also makes learning difficult: changing an early block changes the features on which all later blocks depend.
Classical density theorems establish that neural networks can approximate broad classes of functions \cite{Cybenko1989ApproximationBS,HORNIK1989359,HORNIK1991251}, and quantitative results relate the attainable error to network width, depth, and target regularity \cite{yarotsky2017,yarotsky18a,shijun:Characterized:by:Numer:Neurons,shijun:optimal:rate:in:width:and:depth}.
A typical depth-dependent estimate applies to the final readout of a network constructed for a prescribed depth, as in Figure~\ref{fig:terminal-architecture}.
\blue{A sharper bound at a greater depth may therefore require a different choice of the earlier layers. 
This raises a multiscale question: 
Can we construct a single network one layer at a time, keeping all earlier
parameters and readouts fixed, and bound the approximation error of each
new readout in terms of the network depth and the regularity of the target function?
}

\begin{figure}[!htbp]
\centering
\includegraphics[width=0.985\textwidth]{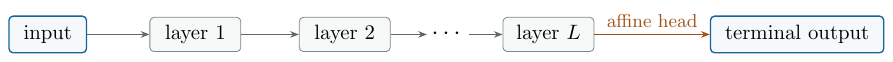}
\caption{Approximation at a prescribed depth. The error bound applies to the affine readout of the last hidden layer.}
\label{fig:terminal-architecture}
\end{figure}

Multigrade deep learning (MGDL) provides a concrete setting for this question \cite{Xu2023,shijun:mgdl:relu:decay}.
At the end of each grade, a scalar readout gives the current approximation. The next grade takes the preceding hidden features as input and fits a correction to the residual, with only the new block and its head trained.
Whether such a correction is possible depends on what those features retain. If two inputs have the same hidden features but different residuals, the next block must assign them the same correction. 
Thus, even when the current readout approximates the target well at one scale,
the underlying features may fail to distinguish inputs that require different
corrections at a finer scale.

Our construction addresses this issue by retaining, at each level, both the information needed to approximate the target on the current grid and the position information needed to identify cells on the next finer grid.
As shown in Figure~\ref{fig:intro-architecture}, the first approximation $\Phi_0$ is available at hidden layer four. Each subsequent layer $\ell+4$ gives two affine readouts, the approximation $\Phi_\ell$ and the correction $\Gamma_\ell$, satisfying
\[
\Phi_\ell=\Phi_{\ell-1}+\Gamma_\ell,\qquad \ell\ge1.
\]
Thus $\Phi_\ell$ can be read directly from the current layer or recovered by adding its correction to the preceding approximation. The additive form is familiar from ResNets \cite{7780459}; here the identity connection carries the scalar approximation.
\blue{Each additional layer refines the grid by a factor of $N$, and the error bound for $\Phi_\ell$ depends on the variation of the target at scale $N^{-\ell}$. All these bounds hold along one sequence of layers, with no prescribed terminal depth.}

\begin{figure}[!htbp]
\centering
\includegraphics[width=0.9985\textwidth]{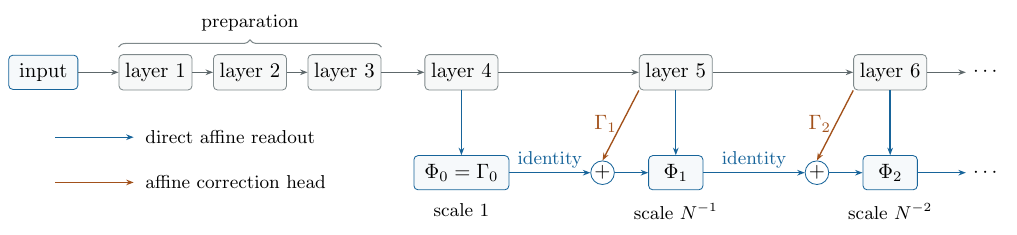}
\caption{Successive readouts of a single network. Layer four gives $\Phi_0$, and each subsequent layer gives both $\Phi_\ell$ and its correction $\Gamma_\ell$.}
\label{fig:intro-architecture}
\end{figure}


\subsection{Main contributions}\label{sec:intro-contributions}

For every real-valued $f\in L^p([0,1]^d)$ with $1\le p<\infty$ and every
integer $N\ge2$, Theorem~\ref{thm:main:Lp} constructs a single network
of width $3dN+2d+4$ without prescribing its terminal depth. Each hidden
layer uses sine in its first two neurons and \ReLU{} in the rest.
The affine readout $\Phi_\ell$ at hidden layer $\ell+4$ satisfies
\begin{equation*}
\|f-\Phi_\ell\|_{L^p([0,1]^d)}
\le (2d+1)\omega_{f,p}(N^{-\ell}),
\qquad \ell=0,1,2,\ldots.
\end{equation*}
Here $\omega_{f,p}(t)$ measures the variation of $f$ in $L^p$ under
translations of length at most $t$, as defined in \eqref{eq:modulus}.
It tends to zero as $t\downarrow0$, even when $f$ has jumps.
Under the $L^p$ H\"older condition $\omega_{f,p}(t)\le\lambda_p t^\alpha$
for $0<t\le1$, with $\lambda_p>0$ and $0<\alpha\le1$,
Corollary~\ref{cor:holder} gives the geometric error bound
$(2d+1)\lambda_p N^{-\alpha\ell}$.

\blue{All the estimates hold for the same sequence of affine maps, with no final depth prescribed. For example, take $N=2$ and a $1$-Lipschitz target. Its translation modulus is at most $t$ at scale $t$, so each additional layer halves the upper bound $(2d+1)2^{-\ell}$. If the current bound meets a tolerance $\varepsilon$ and a tolerance $\varepsilon/4$ is later required, two further layers suffice with the previously chosen maps and readouts fixed.} The depth is chosen from the known regularity bound, and the grid is uniform at each level.

Figure~\ref{fig:grade-path} gives a geometric interpretation when $p=2$. At level $\ell$, the readout lies in the ball centered at $f$ with radius $B_\ell=(2d+1)\omega_{f,2}(N^{-\ell})$. 
These nested bounds do not require the actual approximation errors to decrease at every level.
Our proof chooses the corrections using cell averages of $f$. Whether fitting sampled residuals in MGDL can attain comparable bounds remains open.

\begin{figure}[!htbp]
\centering
\includegraphics[width=0.8\textwidth]{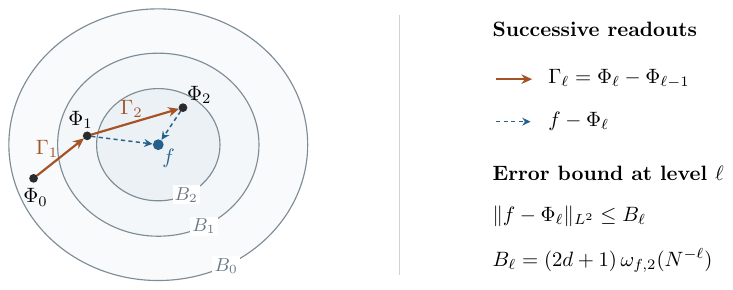}
\caption{A two-dimensional section of the $L^2$ error balls centered at $f$. The solid arrows join illustrative readouts, and the dashed arrows point from a readout to the target. 
}
\label{fig:grade-path}
\end{figure}

For fixed $d$ and $N$, our construction approximates any $1$-Lipschitz target to $L^p$ error $\varepsilon$ using $\calO(\log(1/\varepsilon))$ real affine coefficients as $\varepsilon\downarrow0$. The \ReLU{} neurons identify the current grid cell and retain the information needed to locate finer cells. The sine neurons use the cell label to recover an approximation to the corresponding cell average. At level $\ell$, there are $N^{d\ell}$ such averages. The decoder encodes them in a few real coefficients by allowing a large outer frequency and high coefficient precision. 

\blue{In the over-quantization regime, increasing the number of bits per
parameter in \ReLU{} networks does not improve the order of the minimax
approximation error over a bounded Lipschitz class~\cite{OuSchenkelBolcskei2024Quantization}.
Adding just two sine neurons per layer may allow our construction to exploit the available bit resources to improve accuracy by encoding finer-scale cell averages in a small number of decoder coefficients. Achieving exponential error decay with depth may require more bits than are available in practice.
In practical training, the aim is to improve accuracy within the existing
bit budget, even when exponential error decay with depth is out of reach.
Sinusoidal components have also proved effective in practice, both as hidden activations for representing detailed signals and their derivatives~\cite{NEURIPS2020_53c04118} and as input features for learning high-frequency functions \cite{NEURIPS2020_55053683}. Section~\ref{sec:precision-capacity} discusses the coefficient precision and storage requirements of our construction.}

\subsection{Related work}\label{sec:comparison}

\begin{colorenv}
The closest comparison is with the terminal-depth sine--\ReLU{} result in
\cite{ZZZZ-25-FMMNN}.
Theorem~2.2 and Corollary~2.5 therein give a terminal $L^p$ estimate for continuous targets, with geometric decay in depth for Lipschitz functions.
The encoder is chosen for a prescribed terminal partition, and the decoder fits target values at representative points of its cells.
Intermediate features need not support readouts approximating the corresponding coarser cell values, so increasing the terminal resolution may require changing earlier coefficients.
Remark~2.3 of~\cite{ZZZZ-25-FMMNN}  shows that sine is needed only in the last two hidden layers.

We retain the finite interpolation argument with two sines from
\cite{ZZZZ-25-FMMNN}.
Our additional ingredient is an encoder that supports interpolation at every scale while preserving earlier layers.
It propagates the position within a period exactly and rounds it in separate neurons to obtain the current cell label, allowing transition-strip widths to be chosen independently at each scale.
After initialization, one sine neuron completes the current decoder while the other begins the next.
This produces one additional readout per layer with the same activation pattern throughout.

The multigrade \ReLU{} theorem in
\cite{shijun:mgdl:relu:decay}
provides a complementary result.
For continuous targets on a hypercube, it establishes pointwise nonincreasing residual magnitudes, uniform convergence, and strict decrease of the $L^p$ error whenever the preceding residual is nonzero.
Our construction applies to general $L^p$ targets and gives an explicit error bound at each geometric scale.
These bounds are nonincreasing, although the actual errors at consecutive readouts need not be ordered.
\end{colorenv}
Beyond these direct comparisons, our work connects with the following areas.

\paragraph{Network architecture and activation functions.}
Neural approximation theory studies how approximation rates depend on width, depth, smoothness, and compositional structure
\cite{yarotsky2017,yarotsky18a,B_lcskei_2019,shijun:Characterized:by:Numer:Neurons,shijun:optimal:rate:in:width:and:depth,shijun:smooth:functions,shijun:three:layers,shijun:net:arc:beyond:width:depth,shijun:NonlineArpprox,Pinkus1999MLP}.
Explicit representation results also clarify the effects of architectural restrictions.
Universality of deep convolutional networks is established in
\cite{ZHOU2019}.
The downsampling construction in
\cite{Zhou2020Downsampling}
represents fully connected networks by convolutional networks with the same order of free parameters.
A connection with classical discretization is developed in
\cite{HeLiXuZheng2020},
which represents continuous piecewise linear finite-element functions exactly by deep \ReLU{} networks.
Activation functions provide another way to improve approximation rates
\cite{shijun:2023:beyond:ReLU:to:diverse:actfun,ZZZZ-25-FMMNN}.
Early Fourier neural networks used a cosine squasher activation
\cite{Gallant1988ThereEA}.
Subsequent work developed Fourier features and sinusoidal representations for high-frequency signals, alongside theoretical results for periodic activations
\cite{NEURIPS2020_55053683,NEURIPS2020_53c04118,doi:10.1137/19M1310050,doi:10.1137/21M144431X,pmlr-v139-yarotsky21a}.
For shallow networks, \cite{SIEGEL20221} establishes high-order rates on spectral Barron spaces for cosine and powers of \ReLU{}, together with lower bounds under the stated assumptions.
Sine--\ReLU{} approximation results appear in
\cite{yarotsky:2019:06}, while
\cite{shijun:floor:relu}
combines Floor and \ReLU{} to obtain quantitative terminal-depth estimates for Lipschitz functions.
For some activations, a fixed architecture can approximate every continuous target to arbitrary accuracy using unrestricted real coefficients
\cite{MAIOROV199981,pmlr-v139-yarotsky21a,shijun:arbitrary:error:with:fixed:size,WANG2025107258}.
These results highlight the importance of coefficient precision when interpreting parameter counts.

\paragraph{Multiscale and nonlinear approximation.}
The grid estimate underlying our construction is classical.
Replacing a function by its cell averages gives an error controlled by its variation at the cell scale
\cite{DeVoreLorentz1993}.
Refinement introduces differences between child-cell and parent-cell averages.
The readouts of a single network approximate these piecewise
constant functions and their successive differences.
This connects the construction with nonlinear approximation, where the approximation family and the selection of an approximant for a given target are studied separately
\cite{devore_1998,Ingrid,cohen2020optimal,shijun:intrinsic:parameters,yang2020approximation}.
Here the coefficients depend on the target, and those already chosen remain valid when further layers are appended.

\paragraph{Residual learning and multigrade training.}
Grade-wise learning is related to residual models and layer-wise training
\cite{7780459,Bengio2007,6796673,oreshkin2020nbeats,NEURIPS2023_1d5a9286,9614997}.
Approximation by thin residual networks and connections with dynamical systems are studied in
\cite{NEURIPS2018_03bfc1d4,2026arXiv260315363C,LiLinShen2023DynamicalSystems,doi:10.1137/23M1599744}.
Applications of MGDL to 
regression, integral equations, and PDEs
include
\cite{Xu2023,FangXu2024,FangXu2025,Jiang-Xu2025,XuZeng2023}.
Related studies examine both the order in which different frequencies
are learned and the stability of optimization \cite{rahaman2019spectral,xu2019frequency,xu2019training,luo2019theory,arora2022understanding,cohen2021gradient}.
Our approximation sequence is determined by spatial grid refinement.
Whether training learns these approximations in the same order remains open.

\medskip
The remainder of the paper is organized as follows.
Section~\ref{sec:main_results} states the layer-wise approximation theorem and derives its consequences for regularity, depth, and parameter complexity.
Section~\ref{sec:proof:main} proves the theorem using three propositions on grid approximation, decoding cell averages, and their realization by one sequence of hidden layers.
Section~\ref{sec:proofs-propositions} proves these propositions and details the layer-by-layer construction.
Section~\ref{sec:conclusion} discusses open questions concerning finite precision and multigrade training.

\section{Main results}\label{sec:main_results}
This section presents the layer-wise approximation theorem and its consequences.
After introducing the notation in Section~\ref{sec:notation}, we state the main theorem in Section~\ref{sec:main:thm}, which gives approximation bounds at successive depths along a single sequence of affine maps.
We use these bounds to estimate the depth and number of parameters needed for H\"older targets.
We also extend the main approximation bound to rectangular boxes.
Finally, we discuss coefficient precision, storage cost, and statistical capacity in Section~\ref{sec:precision-capacity}.

\subsection{Notation}\label{sec:notation}

We write
$\N=\{0,1,2,\ldots\}$ and $\N^+=\{1,2,\ldots\}$, while
$\Z$, $\Q$, and $\R$ denote the integers, rationals, and reals.
Bold symbols denote vectors, matrices, and vector-valued maps.
\blue{Unless otherwise specified, vectors are understood to be column vectors; in particular, $(x_1,\ldots,x_n)$ denotes a column vector.}
The Euclidean norm is denoted by $\|\cdot\|_2$, and
$\mathtt{Id}$ denotes the identity map.
For $t\in\R$, $\lfloor t\rfloor$ and $\lceil t\rceil$ denote its floor and ceiling.
We write $\mathds{1}_S$ for the indicator of a set $S$ and $\bme_j$ for
the $j$th standard coordinate vector. 
We use the standard big-$O$ notation $\calO(\cdot)$ for asymptotic upper
bounds, with implicit constants allowed to depend only on fixed parameters.

For a measurable set $Q\subseteq\R^d$ and $1\le p<\infty$,
$L^p(Q)$ is the space of real-valued measurable functions with finite norm
\[
\|f\|_{L^p(Q)}
=\biggl(\int_Q |f(\bmx)|^p\,d\bmx\biggr)^{1/p},
\]
where functions equal almost everywhere are identified. When the domain is clear from the context, we write $\|f\|_{L^p}$ for $\|f\|_{L^p(Q)}$.
We choose a measurable representative whenever pointwise values are needed.
We write \(\lvert Q\rvert\) for the Lebesgue measure of \(Q\).
For $0<|Q|<\infty$, the average of $f\in L^1(Q)$ is
\[
\mathsf{A}_Q(f)=\frac{1}{|Q|}\int_Q f(\bmx)\,d\bmx.
\]
\blue{The sans-serif symbol $\mathsf{A}_Q$ denotes averaging, whereas the calligraphic symbol  $\calA$ denotes an affine map.}
Cell averages and $L^p$ estimates are independent of the chosen representative.
We write $\bar f_\ell$ for the piecewise constant function that equals the
average of $f$ on each full grid cell of side length $N^{-\ell}$ at level $\ell$.

Following \cite{shijun:mgdl:relu:decay}, we denote the affine maps from
$\R^k$ to $\R^m$ by $\aff(k,m)$ and set
\[
\aff_{\le n}=\bigcup_{1\le k,m\le n}\aff(k,m).
\]
A vector-valued affine map is written as
$\bm{\calA}(\bmx)=\bmW\bmx+\bmb$, where
$\bmW\in\R^{m\times k}$ and $\bmb\in\R^m$.
For a scalar-valued affine map, we write
$\calA(\bmx)=\bmW\bmx+b$, with
$\bmW\in\R^{1\times k}$ and $b\in\R$.
For an ordered family of maps, the notation
\[
\ocirc_{i=k}^r\bm T_i
=\bm T_r\circ\cdots\circ\bm T_k
\]
means that $\bm T_k$ acts first.

The construction uses a mixed activation $\bmvarrho$ that applies sine
to the first two coordinates and $\ReLU(t)=\max\{t,0\}$ to the remaining
coordinates. For $\bmx=(x_1,\ldots,x_n)\in\R^n$ with $n\ge2$, define
\begin{equation}\label{eq:mixed-activation}
\bmvarrho(\bmx)
=\bigl(\sin x_1,\sin x_2,\ReLU(x_3),\ldots,\ReLU(x_n)\bigr).
\end{equation}
When $n=2$, this gives $\bmvarrho(\bmx)=(\sin x_1,\sin x_2)$.
When $n=1$, we set $\bmvarrho(x_1)=\sin x_1$.
By a slight abuse of notation, when a single scalar activation is
applied to a vector, we use the same symbol for its coordinatewise action.

For $\bmh\in\R^d$, let
$Q-\bmh=\{\bmx-\bmh:\bmx\in Q\}$.
On the unit cube, the set
\[
E_{\bmh}=[0,1]^d\cap([0,1]^d-\bmh)
\]
consists of the points $\bmx$ for which both $\bmx$ and $\bmx+\bmh$ belong to the cube.
For $f\in L^p([0,1]^d)$, we define its translation modulus by
\begin{equation}\label{eq:modulus}
\omega_{f,p}(t)
=\sup_{\|\bmh\|_2\le t}
\|f(\cdot+\bmh)-f\|_{L^p(E_{\bmh})},
\qquad t\ge0.
\end{equation}
The modulus is nondecreasing and tends to zero as $t\downarrow0$. Indeed, if $\widetilde f$ is the zero extension of $f$ to $\R^d$, then continuity of translations in $L^p(\R^d)$ gives
\begin{equation*}
\omega_{f,p}(t)\le\sup_{\|\bmh\|_2\le t}
\|\widetilde f(\cdot+\bmh)-\widetilde f\|_{L^p(\R^d)}\longrightarrow0.
\end{equation*}
Using the overlap in \eqref{eq:modulus} avoids contributions from the extension across the boundary. In particular, adding a constant to $f$ leaves $\omega_{f,p}$ unchanged.

For $f\in L^\infty([0,1]^d)$, define $\omega_{f,\infty}$ by using the
essential supremum over the overlap in place of the $L^p$ norm
in \eqref{eq:modulus}. Since every overlap has measure at most one,
\begin{equation*}
\omega_{f,p}(t)\le \omega_{f,q}(t)\le \omega_{f,\infty}(t),
\qquad 1\le p\le q<\infty.
\end{equation*}
For continuous $f$, write $\omega_f(t)$ for the classical modulus of
continuity. Then $\omega_{f,\infty}(t)=\omega_f(t)$. This equality
follows by continuity, first for pairs of interior points and then
for arbitrary pairs by moving both points toward the center of the
cube by the same fraction.


For the indicator function $f(\bmx)=\one_{\{x_1>1/2\}}$ and $0<t\le1/2$,
the set on which $f(\cdot+\bm h)$ and $f$ differ in $E_{\bm h}$ has
measure at most $t$ whenever $\|\bm h\|_2\le t$. Its measure is exactly
$t$ for $\bm h=t\bm e_1$. Hence $\omega_{f,p}(t)=t^{1/p}$ and
$\omega_{f,\infty}(t)=1$.
The resulting $L^p$ error bound in \eqref{eq:main-estimate} is $(2d+1)N^{-\ell/p}$ once $N^{-\ell}\le1/2$. Localized variation can have a similar effect for a continuous target. For $0<\varepsilon<1/4$, the triangular spike $f_\varepsilon(\bmx)=\max\{1-|x_1-1/2|/\varepsilon,0\}$ has $\omega_{f_\varepsilon}(\varepsilon)=1$. At that scale the increment is bounded by one and is supported in a strip of width at most $3\varepsilon$, so $\omega_{f_\varepsilon,p}(\varepsilon)\le(3\varepsilon)^{1/p}$.



\subsection{Main theorem and its consequences}\label{sec:main:thm}


With the notation above, we can now state the main approximation theorem.

\begin{colorenv}
\begin{theorem}\label{thm:main:Lp}
Let $f\in L^p([0,1]^d)$ with $1\le p<\infty$, let $N\ge2$ be an integer, and let $\bmvarrho$ be the mixed activation in \eqref{eq:mixed-activation}.
There exist affine maps 
$\bm{\calA}_r\in\aff_{\le 3dN+2d+4}$ for $r\ge1$ and
$\calA_j^{\mathrm{out}}\in\aff(3dN+2d+4,1)$ for $j\ge0$ such that the functions
\begin{equation*}
\Phi_\ell
=\sum_{j=0}^{\ell}\calA_j^{\mathrm{out}}
\circ\ocirc_{r=1}^{j+4}\bigl(\bmvarrho\circ\bm{\calA}_r\bigr)
\end{equation*}
satisfy
\begin{equation}\label{eq:main-estimate}
\|f-\Phi_\ell\|_{L^p([0,1]^d)}
\le (2d+1)\omega_{f,p}(N^{-\ell})
\qquad\text{for all }\ell\ge0.
\end{equation}
\end{theorem}

The proof of Theorem~\ref{thm:main:Lp} is given in Section~\ref{sec:infinite-compatibility}.
Each composition $\bmvarrho\circ\bm{\calA}_r$ represents one hidden
layer of a neural network. By the definition of $\bmvarrho$, the
first two neurons use sine and all remaining neurons use \ReLU{}.
Each hidden affine map has $3dN+2d+4$ output coordinates,
so every hidden layer has this width.
Width counts all hidden neurons, including those used to retain
input coordinates or previous values, and depth counts hidden layers.
The approximation $\Phi_\ell$ uses the first $\ell+4$ hidden layers.
Thus four layers give the first approximation, and each additional
layer supplies an approximation at the next scale.
Although $\Phi_\ell$ is written as a sum in the theorem, it is also
a scalar affine function of
$\ocirc_{r=1}^{\ell+4}(\bmvarrho\circ\bm{\calA}_r)$.
Figure~\ref{fig:intro-architecture} illustrates how the readouts are attached to one sequence of hidden layers.
\end{colorenv}

\subsubsection{Layer-wise approximations and corrections}
To give the formulas for these readouts, write
$\bm{\calA}_r(\bmy)=\bmW_r\bmy+\bmb_r$ and define the
\emph{hidden-state maps} by
\begin{equation*}
\bm{\mathcal V}_0(\bmx)=\bmx,\qquad
\bm{\mathcal V}_r
=
\bmvarrho\circ\bm{\calA}_r\circ\bm{\mathcal V}_{r-1},
\qquad r\ge1.
\end{equation*}
Then
$\bm{\mathcal V}_r
=\ocirc_{i=1}^{r}(\bmvarrho\circ\bm{\calA}_i)$,
and $\bm{\mathcal V}_r(\bmx)$ is the vector of neuron values in
hidden layer $r$ for the input $\bmx$.
Write $\Gamma_0=\Phi_0$ and $\Gamma_j=\Phi_j-\Phi_{j-1}$ for $j\ge1$.
The scalar affine map $\calA_j^{\mathrm{out}}(\bmy)
=\bmW_j^{\mathrm{out}}\bmy+b_j^{\mathrm{out}}$ then gives
\begin{equation}\label{eq:prefix-readout}
\Gamma_j=\calA_j^{\mathrm{out}}\circ\bm{\mathcal V}_{j+4},\qquad
\Phi_\ell=\sum_{j=0}^\ell\Gamma_j .
\end{equation}
At hidden layer $\ell+4$, the state supplies both the current approximation $\Phi_\ell$ and the correction $\Gamma_\ell$ in \eqref{eq:prefix-readout}. Passing to layer $\ell+5$ gives the next approximation with its bound at scale $N^{-(\ell+1)}$, while all earlier layers and readouts remain fixed.

\blue{Let $r_0=f$ and $r_{\ell+1}=f-\Phi_\ell$ for $\ell\ge0$. The residual entering grade $\ell$ is $r_\ell$. Subtracting its new correction gives the residual after the update,
\begin{equation*}
r_{\ell+1}=r_\ell-\Gamma_\ell,\qquad
\|f-\Phi_\ell\|_{L^p}=\|r_\ell-\Gamma_\ell\|_{L^p}.
\end{equation*}
For $\ell\ge1$, the proof constructs $\Gamma_\ell$ by subtracting two network readouts fitted to the cell averages of the original target at successive levels.}
At the first two levels, the distinction between the correction and the approximation takes the form
\begin{equation*}
\Phi_0=\calA_0^{\mathrm{out}}\circ\bm{\mathcal V}_4,\qquad
\Phi_1=\Phi_0+\calA_1^{\mathrm{out}}\circ\bm{\mathcal V}_5.
\end{equation*}
The head at layer five supplies the increment to $\Phi_0$. The construction also gives an affine function of $\bm{\mathcal V}_5$ alone whose value is $\Phi_1$. The proof arranges these two readouts using neurons already counted in the stated width.

\blue{Once $f$, $p$, and $N$ are fixed, we may stop at any level $m$ and later continue with the same affine maps through layer $m+4$ and the same correction heads through level $m$.} Each $\Phi_\ell$ is continuous, and \eqref{eq:main-estimate} gives $\Phi_\ell\to f$ in $L^p$ even when $f$ is discontinuous. The bounds are nonincreasing and tend to zero as $\ell$ grows, although the actual errors may increase at some steps. Corollary~\ref{cor:holder} gives a geometric rate when the modulus decays algebraically.


\subsubsection{Approximation rates and depth selection under H\"older regularity}\label{sec:holder_complexity}

Under an $L^p$ H\"older condition, the exponent $\alpha$ determines how
the approximation bound improves with depth. At level $\ell$, the grid
scale is $N^{-\ell}$, so evaluating the H\"older bound for the
translation modulus at this scale gives the following geometric rate.

\begin{corollary}\label{cor:holder}
Let $f$, $N$, and $\Phi_\ell$ be as in Theorem~\ref{thm:main:Lp}.
Suppose that $\omega_{f,p}(t)\le\lambda_p t^\alpha$ for $0<t\le1$,
where $\lambda_p>0$ and $0<\alpha\le1$. Then
\begin{equation*}
\|f-\Phi_\ell\|_{L^p([0,1]^d)}
\le(2d+1)\lambda_p N^{-\alpha\ell},
\qquad \ell\ge0.
\end{equation*}
\end{corollary}

Each additional hidden layer multiplies this upper bound by
$N^{-\alpha}$. Increasing $N$ gives a finer grid at each level at the
cost of a wider network. For a pointwise $\alpha$-H\"older target with
constant $\lambda$, the modulus bound holds with $\lambda_p=\lambda$,
since
\[
\|f(\cdot+\bmh)-f\|_{L^p(E_{\bmh})}
\le \lambda |E_{\bmh}|^{1/p}\|\bmh\|_2^\alpha
\le \lambda\|\bmh\|_2^\alpha.
\]
\blue{For a 1-Lipschitz target, the approximation error at hidden layer \(\ell+4\) is at most \((2d+1)N^{-\ell}\).}

For a $1$-Lipschitz target, the constant can be sharpened. On each cell
$Q$ of side $N^{-\ell}$,
\[
|f(\bmx)-\mathsf{A}_Q(f)|
\le \frac{1}{|Q|}\int_Q\|\bmx-\bmy\|_2\,d\bmy
\le \sqrt d\,N^{-\ell},\qquad \bmx\in Q.
\]
Thus $\|f-\bar f_\ell\|_{L^p}\le\sqrt d\,N^{-\ell}$. Given $\tau>0$
and a nonconstant target, choose the decoding tolerance at level $\ell$
to be $\min\{\tau N^{-\ell},\omega_{f,p}(N^{-\ell})\}$. The construction
then gives one sequence satisfying both~\eqref{eq:main-estimate} and
$\|f-\Phi_\ell\|_{L^p}\le(\sqrt d+\tau)N^{-\ell}$ for every $\ell$.
Constant functions are represented exactly. 
We do not seek a sharp constant here. The optimal width and initial
depth remain open.

For comparison, the sine--\ReLU{} network in \cite{ZZZZ-25-FMMNN} has
width $d(4N-1)$ and $L+2$ hidden layers, with terminal $L^p$ error at
most $2\sqrt d\,\omega_f(N^{-L})$ for continuous targets. The
Floor--\ReLU{} construction in \cite{shijun:floor:relu} uses width
$\max\{d,5N+13\}$ and depth $64dL+3$, with terminal error at most
$3\sqrt d\,N^{-\sqrt L}$ for $1$-Lipschitz targets.
Theorem~\ref{thm:main:Lp} gives bounds at all levels of one sequence.
Under the H\"older assumption of Corollary~\ref{cor:holder},
$\lceil 1/\alpha\rceil$ additional hidden layers suffice to reduce the
error bound by a factor of at least $N$. For a prescribed tolerance
$\varepsilon$, solving the same bound for the level gives the following
depth estimate.

\begin{corollary}\label{cor:depth}
Under the assumptions of Corollary~\ref{cor:holder}, let $\varepsilon>0$
and choose any integer $m\ge 0$ satisfying
\[
m\ge\frac{1}{\alpha}\log_N\frac{(2d+1)\lambda_p}{\varepsilon}.
\]
Then $\|f-\Phi_m\|_{L^p([0,1]^d)}\le\varepsilon$. The first $m+4$ hidden
layers and the correction heads through level $m$ use
$\calO\bigl((3dN+2d+4)^2(m+4)\bigr)$ 
real coefficients in total.
\end{corollary}

To see the parameter count, write $\widetilde N=3dN+2d+4$. The first
hidden layer, the next $m+3$ hidden layers, and the $m+1$ scalar heads
contribute, respectively, the three terms in
\begin{equation*}
\widetilde N(d+1)+(m+3)\widetilde N(\widetilde N+1)
+(m+1)(\widetilde N+1).
\end{equation*}
This dense count includes biases, zero entries, and coefficients used to
copy or store values. With the smallest admissible $m$, the initial four
layers suffice when $\varepsilon\ge(2d+1)\lambda_p$. For fixed $N$ and
$\alpha$, the count is
$O\bigl(d^2[1+\log((2d+1)\lambda_p/\varepsilon)]\bigr)$
as $\varepsilon\downarrow0$.
\blue{This a priori choice of depth uses the given H\"older exponent
$\alpha$ and modulus bound $\omega_{f,p}(t)\le\lambda_p t^\alpha$.}
For a target observed only through samples, one would still need to
estimate its regularity or its residual.

\subsubsection{Approximation on rectangular boxes}\label{sec:rectangular-boxes}

Let $Q=\prod_{i=1}^d[a_i,b_i]$ with $a_i<b_i$, and define
\[
E_{\bmh}^Q=\{\bmx\in Q:\bmx+\bmh\in Q\},\qquad
\omega_{f,p}^Q(t)=\sup_{\|\bmh\|_2\le t}
\|f(\cdot+\bmh)-f\|_{L^p(E_{\bmh}^Q)}.
\]
We omit the superscript when $Q=[0,1]^d$.
\blue{Rescaling the unit cube to $Q$ sends a displacement of length at
most $t$ to one of length at most $s_Qt$, 
where \(s_Q=\max_i(b_i-a_i)\) is the length of the longest side of \(Q\).
}

\begin{colorenv}
\begin{theorem}\label{thm:box}
Let $f\in L^p(Q)$ with $1\le p<\infty$, and let $N\ge2$ be an integer.
There is one sequence of affine maps and scalar heads as in
Theorem~\ref{thm:main:Lp} whose approximations on $Q$ satisfy
\[
\|f-\Phi_\ell\|_{L^p(Q)}
\le(2d+1)\omega_{f,p}^Q(s_QN^{-\ell}),\qquad \ell\ge0.
\]
\end{theorem}
\end{colorenv}

\begin{proof}
Let $\bm T_Q(\bmy)=\bm a+\bm D_Q\bmy$, where
$\bm a=(a_1,\ldots,a_d)$ and $\bm D_Q$ is the diagonal matrix
with diagonal entries $b_1-a_1,\ldots,b_d-a_d$.
For $g=f\circ\bm T_Q$, a change of variables on the overlap gives
\[
\|g(\cdot+\bmh)-g\|_{L^p(E_{\bmh})}
=|Q|^{-1/p}
\|f(\cdot+\bm D_Q\bmh)-f\|_{L^p(E^Q_{\bm D_Q\bmh})}.
\]
Since $\|\bm D_Q\bmh\|_2\le s_Q\|\bmh\|_2$, taking the supremum gives
$\omega_{g,p}(t)\le |Q|^{-1/p}\omega_{f,p}^Q(s_Qt)$.

Apply Theorem~\ref{thm:main:Lp} to $g$ and denote its readouts by
$\widetilde\Phi_\ell$. Composing the first affine map with $\bm T_Q^{-1}$
realizes $\Phi_\ell=\widetilde\Phi_\ell\circ\bm T_Q^{-1}$ on $Q$ for
every $\ell$. Hence
\[
\|f-\Phi_\ell\|_{L^p(Q)}
=|Q|^{1/p}\|g-\widetilde\Phi_\ell\|_{L^p([0,1]^d)}
\le(2d+1)\omega_{f,p}^Q(s_QN^{-\ell}).
\]
Since the composition affects only the first affine map, the same network
architecture realizes all the readouts $\Phi_\ell$, completing the proof.
\end{proof}

\blue{For a pointwise $\alpha$-H\"older target with constant $\lambda$,
the bound $|E^Q_{\bmh}|\le|Q|$ gives
$\omega_{f,p}^Q(t)\le\lambda|Q|^{1/p}t^\alpha$.} It follows that
\[
\|f-\Phi_\ell\|_{L^p(Q)}
\le(2d+1)\lambda|Q|^{1/p}s_Q^\alpha N^{-\alpha\ell}.
\]
Thus the H\"older exponent gives the same geometric factor
$N^{-\alpha}$ per additional layer as on the unit cube. For the
uniform probability measure on $Q$, the factor $|Q|^{1/p}$ drops out
and the bound becomes $(2d+1)\lambda s_Q^\alpha N^{-\alpha\ell}$.

\subsection{Coefficient precision and statistical capacity}\label{sec:precision-capacity}

Corollary~\ref{cor:depth} counts real coefficients regardless of their
precision. \blue{For a fixed \ReLU{} architecture, additional bits per
parameter can cease to improve the worst-case approximation rate. 
We add
just two sine neurons per hidden layer in the hope of putting the available bits to
better use by encoding finer-scale cell averages. 
In practice, training
need not achieve exponential convergence with depth. The aim is to
improve accuracy within the existing bit budget.}

At level $\ell$, let $K=N^{d\ell}$ be the number of grid cells and
let $u_\ell$ be the largest absolute value of their averages.
At each integer label $k=1,\ldots,K$, the decoder returns
$u_\ell\sin(v_\ell\sin(w_\ell k))$.
Lemma~\ref{prop:k:to:yk} allows us to fix the inner frequency $w_\ell$ using only $K$
and then choose an integer outer frequency $v_\ell$ to fit the
prescribed values.
\blue{To see the
potential cost, take $u_\ell=1$ and prescribe an arbitrary sign at
each label. An error strictly less than one at every label recovers
all $K$ signs. Encoding all $2^K$ possible patterns therefore requires
at least $K=N^{d\ell}$ bits for $v_\ell$ in a fixed-length encoding.}
The lemma gives no quantitative upper bound on the magnitude of $v_\ell$
needed for a given tolerance. Deriving such a bound is beyond the scope
of this paper. For specific activation functions, related fixed-depth estimates quantify
the tradeoff between width and coefficient range~\cite{LiShenYangZhang2026Tradeoffs}.

\begin{colorenv}
The classical entropy estimate of Kolmogorov and Tikhomirov
\cite{KolmogorovTikhomirov1959} gives metric entropy of order
$\varepsilon^{-d}$ for the bounded unit Lipschitz class on $[0,1]^d$
at uniform-norm accuracy $\varepsilon\downarrow0$.
For all sufficiently small $\varepsilon>0$, the $L^1$ packing bound
in~\cite[Section~3.1, Claims~1--3]{MaiorovRatsaby1999} implies that
the $1$-Lipschitz class on $[0,1]^d$ with $\|f\|_{L^\infty}\le1$
contains at least $2^{c_d\varepsilon^{-d}}$ functions with pairwise
$L^1$ distances greater than $2\varepsilon$, where $c_d>0$ depends
only on $d$. Since $\|g\|_{L^1}\le\|g\|_{L^p}$ on the unit cube
for $1\le p<\infty$, these functions require distinct descriptions
at $L^p$ accuracy $\varepsilon$. A fixed architecture with $P$
coefficients and $q$ bits per coefficient therefore requires
\[
Pq\ge c_d\varepsilon^{-d}.
\]
For fixed $d,p,N$, the parameter count $P=\calO(\log(1/\varepsilon))$
of Corollary~\ref{cor:depth} thus requires a common bit width of at
least $c\,\varepsilon^{-d}/\log(1/\varepsilon)$ in the worst case,
where $c>0$ is independent of $\varepsilon$.
\end{colorenv}
For related encoding bounds for sparse networks, see~\cite{B_lcskei_2019}.

\blue{A fixed-length encoding of $q_{\mathrm{tot}}$ bits describes at
most $2^{q_{\mathrm{tot}}}$ networks, so their zero-threshold
classifiers have VC-dimension at most $q_{\mathrm{tot}}$. For
multigrade training, such capacity control is relevant to prediction
bounds for corrections fitted to sampled residuals. Related covering
estimates for \ReLU{} networks with bounded or quantized parameters are
given in~\cite{OuBolcskei2026Covering}.}

\section{Proof of Theorem~\ref{thm:main:Lp}}\label{sec:proof:main}\label{sec:preparations:proof:main}

We prove Theorem~\ref{thm:main:Lp} by constructing a single network whose readouts satisfy the stated error bound at every
level. Section~\ref{sec:proof-strategy} illustrates the refinement of
cell averages on a dyadic grid and explains how the averaging and
decoding errors enter the proof. Section~\ref{sec:auxiliary-statements}
states three propositions that bound the averaging error, approximate
the cell averages by a continuous decoder, and realize the resulting
approximations and their successive differences as affine readouts.
Finally, Section~\ref{sec:infinite-compatibility} combines these
propositions by choosing the decoding tolerance at each level to match
the translation modulus at that scale.

\subsection{Proof strategy}\label{sec:proof-strategy}

Figure~\ref{fig:main-proof-idea} illustrates the refinement for $f(x)=\sqrt{x}$ and $N=2$.
Let \(\bar f_\ell\) be the piecewise constant function equal to the mean of \(f\) on each dyadic interval of length \(2^{-\ell}\).
The mean on $[0,1/2)$ is $\sqrt{2}/3$, and the means on its two halves are
\begin{equation*}
4\int_0^{1/4}\sqrt{x}dx=\frac13,
\qquad
4\int_{1/4}^{1/2}\sqrt{x}dx=\frac{2\sqrt{2}-1}{3}.
\end{equation*}
Thus $\bar f_2-\bar f_1$ equals $-(\sqrt{2}-1)/3$ on the first quarter and $(\sqrt{2}-1)/3$ on the second. More generally, the child cells have equal volume, so their means average to the parent mean. The difference of successive network readouts approximates this change in cell averages. In this example, averaging $|\sqrt{x}-\sqrt{y}|\le\sqrt{|x-y|}$ over $y$ in an interval of length $s$ shows that the average approximates $f$ there with error at most $\sqrt{s}$. For general $f$, Proposition~\ref{prop:f-aver:p:norm:upper_bound} bounds the grid error in $L^p$ by the translation modulus at the cell scale.

\begin{figure}[!htbp]
\centering
\includegraphics[width=0.88\textwidth]{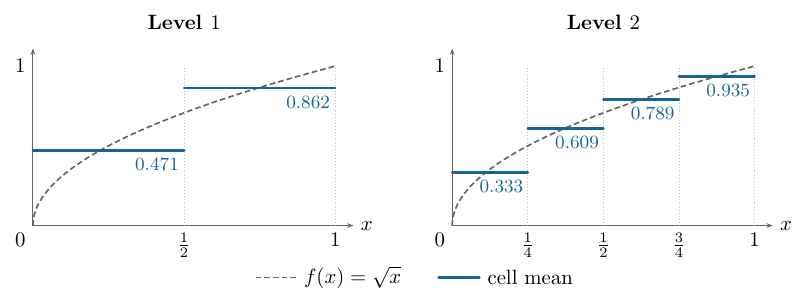}
\caption{Cell averages for $f(x)=\sqrt{x}$ on two successive dyadic grids. Each horizontal segment gives the mean on its interval. On each parent interval, the two child means differ from the parent mean by equal and opposite amounts.}
\label{fig:main-proof-idea}
\end{figure}

\begin{figure}[!htbp]
\centering
\includegraphics[width=0.985\textwidth]{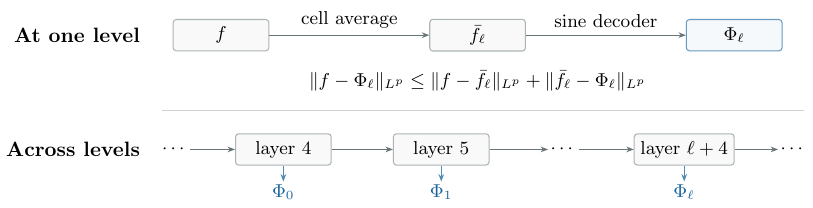}
\caption{Proof of Theorem~\ref{thm:main:Lp}. At a fixed level,
$\bar f_\ell$ approximates $f$ by cell averages, and $\Phi_\ell$
approximates those averages by a continuous readout. Across levels,
the readouts come from successive layers of the same network.}
\label{fig:proof-roadmap}
\end{figure}

At level $\ell$, the grid cells have side $N^{-\ell}$. As illustrated in Figure~\ref{fig:proof-roadmap}, the triangle inequality bounds the approximation error by the sum of the averaging and decoding errors, giving
\begin{equation*}
\|f-\Phi_\ell\|_{L^p}
\le \|f-\bar f_\ell\|_{L^p}
   +\|\bar f_\ell-\Phi_\ell\|_{L^p}.
\end{equation*}
To approximate the piecewise constant function $\bar f_\ell$ by a
continuous readout, we assign an integer label to each decoding box
and fit the corresponding full-cell mean at that label using two
sine evaluations. The label is constant on each decoding box and
varies continuously across the transition strips.
The full-cell means determine the decoder amplitude. Once this amplitude is fixed, Proposition~\ref{prop:decoded-averages} allows us to choose the frequencies and strip width so that $\|\bar f_\ell-\Phi_\ell\|_{L^p}$ is as small as prescribed.

Across levels, the readouts must come from one sequence of affine maps. A rounded coordinate identifies the current cell but loses the position needed to locate a finer cell. The network therefore retains the input and propagates the position within each coordinate interval as an exact periodic phase. A separate branch uses the phase and the input to compute the current label, leaving the phase available for the next refinement. The transition width can thus be chosen independently at each level.

Proposition~\ref{lem:pipeline} realizes these readouts with the prescribed mixed activation. The two sine evaluations at each scale take place in consecutive layers. The layer that completes one evaluation pair also starts the next, so the first approximation requires four hidden layers and each subsequent approximation requires one additional layer. Two \ReLU{} neurons retain the positive and negative parts of the preceding approximation, allowing an affine head to recover the difference between consecutive approximations.
Section~\ref{sec:auxiliary-statements} describes the three propositions
used in this construction, and Section~\ref{sec:infinite-compatibility}
combines their estimates. Section~\ref{sec:proofs-propositions} verifies the constructions.

\subsection{Cell averages, decoding, and layer-wise realization}\label{sec:auxiliary-statements}

Fix $d\in\N^+$, $1\le p<\infty$, $f\in L^p([0,1]^d)$, and an integer $N\ge2$.
\blue{Proposition~\ref{prop:f-aver:p:norm:upper_bound} bounds the error of cell averaging, and Proposition~\ref{prop:decoded-averages} approximates the cell averages to any prescribed $L^p$ tolerance. Proposition~\ref{lem:pipeline} realizes the resulting functions exactly as readouts of one network, so this last step adds no approximation error.} The proofs of Propositions~\ref{prop:f-aver:p:norm:upper_bound},
\ref{prop:decoded-averages}, and~\ref{lem:pipeline} are given in
Sections~\ref{sec:proof-cell-averages}, \ref{sec:levelwise-data},
and~\ref{sec:exact-phase}, respectively.

Figure~\ref{fig:grid-cell-averages} distinguishes the full cell used for averaging from the smaller box on which the decoder fits that average.
At level $\ell\ge0$, put $s=N^{-\ell}$. The full cells are
products of intervals of length $s$. More precisely, let
\begin{equation*}
\widehat{\calI}_{\ell,j}=[js,(j+1)s),
\qquad j=0,\ldots,N^\ell-1,
\end{equation*}
\begin{equation*}
\widehat Q_{\ell,\bmbeta}
=\prod_{j=1}^d\widehat{\calI}_{\ell,\beta_j},
\qquad
\bmbeta\in\{0,\ldots,N^\ell-1\}^d.
\end{equation*}
Since $p\ge1$ and the cube has finite measure, $f$ is integrable on every cell. With $\mathsf{A}_Q(f)$ denoting the average over $Q$, set
\begin{equation}\label{eq:cell-average-grid}
\bar f_\ell(\bmx)=\mathsf{A}_{\widehat Q_{\ell,\bmbeta}}(f),
\qquad \bmx\in\widehat Q_{\ell,\bmbeta}.
\end{equation}
The outer faces omitted by the half-open convention have measure
zero, so $\bar f_\ell$ may be defined arbitrarily there. 
At level zero, \(\bar f_0=\mathsf A_{[0,1]^d}(f)\) almost everywhere.
At every refinement, each cell splits into $N^d$ smaller
cells, and we take the average of $f$ on each of them.

\begin{figure}[!htbp]
\centering
\includegraphics[width=0.99\textwidth]{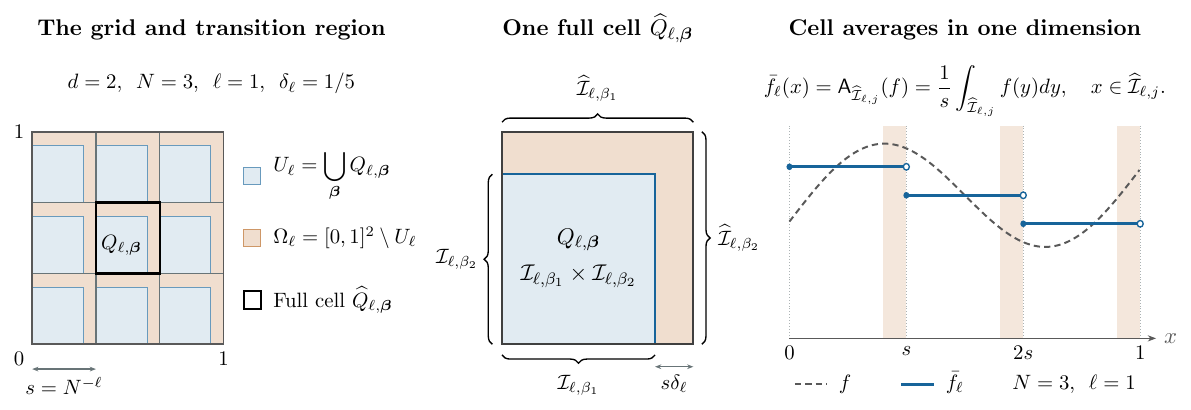}
\caption{The left panel shows the grid. The blue decoding boxes form
$U_\ell$, and the orange strips indicate the transition region
$\Omega_\ell$ up to boundary sets. The middle panel enlarges the
outlined cell $\widehat Q_{\ell,\bmbeta}$ and shows its decoding box
$Q_{\ell,\bmbeta}$. In the one-dimensional example on the right, the
dashed curve represents $f$, and each blue segment gives the value of
$\bar f_\ell$ on one full interval $\widehat{\calI}_{\ell,j}$, obtained by
averaging $f$ over that interval, including its transition strip.}
\label{fig:grid-cell-averages}
\end{figure}

\begin{colorenv}
Neighboring cells may have different averages. To allow a continuous
readout to pass from one value to the next, we reserve a narrow
strip at the right end of each coordinate interval. Choose
$\delta_\ell\in(0,1/3)$ and define
\begin{equation*}
\calI_{\ell,j}=[js,(j+1-\delta_\ell)s],
\qquad
Q_{\ell,\bmbeta}=\prod_{j=1}^d\calI_{\ell,\beta_j}.
\end{equation*}
On the smaller box $Q_{\ell,\bmbeta}$, the decoder will approximate
the average over the full cell $\widehat Q_{\ell,\bmbeta}$. The
union of the smaller boxes and the region left for transitions are
denoted by
\begin{equation*}
U_\ell=\bigcup_{\bmbeta}Q_{\ell,\bmbeta},
\qquad
\Omega_\ell=[0,1]^d\setminus U_\ell.
\end{equation*}
Each transition strip has width $s\delta_\ell$, as shown in the middle panel of Figure~\ref{fig:grid-cell-averages}. Varying $\delta_\ell$ changes the decoding boxes while preserving the full-cell averages.
\end{colorenv}
A cell of side $s$ has diameter $\sqrt d\,s$. To estimate the grid error using $\omega_{f,p}(s)$, we average one coordinate at a time. Each coordinate average uses displacements of length at most $s$, and summing the $d$ contributions gives the factor $d$ below.

\begin{proposition}\label{prop:f-aver:p:norm:upper_bound}
Let $\bar f_\ell$ be the function of full-cell averages in
\eqref{eq:cell-average-grid}. Then, for all $\ell\ge0$,
\begin{equation}\label{eq:projection-error}
\|f-\bar f_\ell\|_{L^p([0,1]^d)}
\le 2^{1/p}d\,\omega_{f,p}(N^{-\ell}).
\end{equation}
\end{proposition}

Section~\ref{sec:proof-cell-averages} proves this estimate. To decode the cell averages, we next assign a distinct integer label to each box $Q_{\ell,\bmbeta}$. The integer parts of $N^\ell x_1,\ldots,N^\ell x_d$ identify the full cell containing $\bmx\in[0,1)^d$. To obtain a continuous label, we replace the jumps at grid boundaries by linear transitions in $\Omega_\ell$.
Write $x=(m+t)s$ with $m=0,\ldots,N^\ell-1$ and $0\le t\le1$,
so $t$ measures the position of $x$ within its grid interval.
We round $x$ to the left endpoint $ms$ for $0\le t\le1-\delta_\ell$
and then interpolate linearly to the next endpoint. The resulting
function is
\begin{equation}\label{eq:continuous-rounding}
\psi_\ell\bigl((m+t)s\bigr)
=s\biggl(m+\frac{\ReLU(t-1+\delta_\ell)}{\delta_\ell}\biggr),
\qquad 0\le t\le1 .
\end{equation}
The left panel of Figure~\ref{fig:continuous-rounding} shows the
constant part on $\calI_{\ell,m}$ and the linear rise across the
remaining distance $s\delta_\ell$. At a shared grid point, taking
$t=1$ on the left and $t=0$ on the right gives the same value.
Thus $\psi_\ell$ is continuous on $[0,1]$, with $\psi_\ell(1)=1$.

\begin{figure}[htbp]
  \centering
\includegraphics[width=0.95\textwidth]{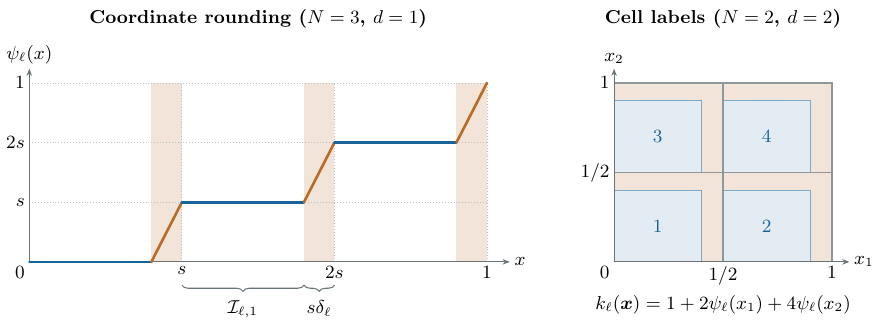}
  \caption{Coordinate rounding and cell labels at level $\ell=1$ with
    $\delta_\ell=1/5$. On the left, $N=3$ and $\psi_\ell$ rises only
    on the orange strips, each of width $s\delta_\ell$. On the right,
    $N=2$ and the two rounded coordinates give the label shown in each
    blue box.}
  \label{fig:continuous-rounding}
\end{figure}

To assign an integer label to each decoding box, introduce the scalar affine
map $\calA_\ell^{\mathrm{lab}}$ by
\begin{equation*}
\calA_\ell^{\mathrm{lab}}(\bmy)=1+\sum_{j=1}^d N^{j\ell}y_j,
\qquad \bmy=(y_1,\ldots,y_d)\in\R^d.
\end{equation*}
Applying this map to the rounded coordinates of \(\bm x\) gives the label
\begin{equation}\label{eq:scalar-label}
k_\ell(\bmx)
=\calA_\ell^{\mathrm{lab}}\bigl(\psi_\ell(x_1),\ldots,\psi_\ell(x_d)\bigr).
\end{equation}
For $\bmx\in Q_{\ell,\bmbeta}$, we have $\psi_\ell(x_j)=\beta_j N^{-\ell}$ for each $j$. Hence $k_\ell$ equals
\begin{equation*}
1+\beta_1+N^\ell\beta_2+\cdots+N^{(d-1)\ell}\beta_d .
\end{equation*}
For $\ell\ge1$, the digits $\beta_j$ range from $0$ to $N^\ell-1$, so this formula numbers the boxes from $1$ to $N^{d\ell}$. At level zero, the single box has label $1$. In the example $d=N=2$ and $\ell=1$ in Figure~\ref{fig:continuous-rounding}, the boxes indexed by $(0,0),(1,0),(0,1),(1,1)$ have labels $1,2,3,4$, respectively. On $\Omega_\ell$, the label varies continuously and need not be an integer.

At the integer label assigned to $Q_{\ell,\bmbeta}$, the decoder must fit the mean on $\widehat Q_{\ell,\bmbeta}$. 
Define
\begin{equation}\label{eq:decoded-value}
\Phi_\ell(\bmx)\coloneqq  u_\ell\sin\Bigl(v_\ell\sin\bigl(w_\ell k_\ell(\bmx)\bigr)\Bigr),\qquad \tn{where}\  u_\ell=\max_{\bmbeta}|\mathsf{A}_{\widehat Q_{\ell,\bmbeta}}(f)|.
\end{equation}
The frequencies $v_\ell,w_\ell$ are shared by all labels at level $\ell$ and are chosen to fit the full-cell means at those labels.
\blue{On the transition strips the decoder may oscillate rapidly, but both $\Phi_\ell$ and $\bar f_\ell$ are bounded in magnitude by $u_\ell$ almost everywhere. The strips therefore contribute at most $2u_\ell|\Omega_\ell|^{1/p}$ to the $L^p$ error. Narrowing them reduces this bound while preserving the values to be fitted.}

\begin{colorenv}
\begin{proposition}\label{prop:decoded-averages}
For any $\ell\ge0$ and $\eta_\ell>0$, let $u_\ell$ be the largest
absolute full-cell average. There are $\delta_\ell\in(0,1/3)$ and
$v_\ell,w_\ell\in\R$ such that the continuous function
$\Phi_\ell$ in \eqref{eq:decoded-value} satisfies
\begin{equation}\label{eq:decoded-estimate}
\|\bar f_\ell-\Phi_\ell\|_{L^p([0,1]^d)}\le\eta_\ell .
\end{equation}
\end{proposition}
\end{colorenv}

We remark that when adjacent cell averages differ, continuity of the readout
forces the larger one-sided limiting error at their common face to be at least
half the jump in the averages. For finite $p$, narrowing the transition region
makes its contribution to the $L^p$ error arbitrarily small.
Section~\ref{sec:levelwise-data} gives the parameter choices.

To obtain the multigrade form in \eqref{eq:prefix-readout}, set $\Gamma_0=\Phi_0$ and $\Gamma_\ell=\Phi_\ell-\Phi_{\ell-1}$ for $\ell\ge1$. The corresponding grid correction on a child cell is its mean minus the parent mean. Since
\begin{equation*}
\Gamma_\ell-(\bar f_\ell-\bar f_{\ell-1})
=(\Phi_\ell-\bar f_\ell)-(\Phi_{\ell-1}-\bar f_{\ell-1}),
\end{equation*}
the triangle inequality and \eqref{eq:decoded-estimate} give
\begin{equation*}
\bigl\|\Gamma_\ell-(\bar f_\ell-\bar f_{\ell-1})\bigr\|_{L^p}
\le \eta_\ell+\eta_{\ell-1},
\qquad \ell\ge1 .
\end{equation*}
For the cumulative approximation $\Phi_\ell$ itself, a direct comparison with $\bar f_\ell$ involves only the current decoding error.




To distinguish finer cells, the hidden state must retain more than the
current label. Figure~\ref{fig:intro-refinement} illustrates this with an example
in which $d=1$, $N=4$, $\delta_0=1/5$, and $\delta_1=1/10$. The points $3/10$ and $7/10$
both have label $1$ at level zero, but their level-one labels are $2$
and $3$, respectively. The coarse label alone therefore cannot
determine the finer label.

\begin{figure}[htbp]
\centering
\includegraphics[width=0.84\textwidth]{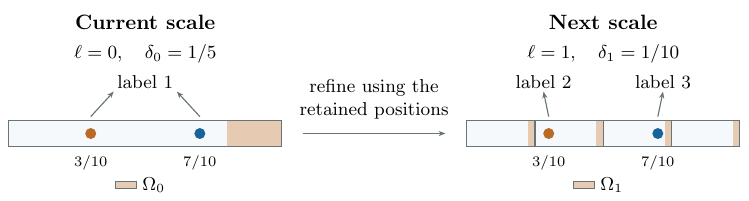}
\caption{The orange strips indicate the transition regions $\Omega_0$ and
$\Omega_1$. Both points lie in decoding intervals at both levels. They
share a coarse label, but the retained phase distinguishes their positions.}
\label{fig:intro-refinement}
\end{figure}

\blue{The hidden state retains the input coordinates and encodes their
positions within the current cells as periodic phases. The next layer
refines these phases exactly and uses the current phases for rounding.
Separate coordinates store the preceding approximation, allowing an
affine head to subtract it from the current one.}
The following proposition applies to any prescribed sequence of decoder
parameters, including those in Proposition~\ref{prop:decoded-averages}.

\begin{colorenv}
\begin{proposition}\label{lem:pipeline}
For all $\ell\ge0$, fix $\delta_\ell\in(0,1/3)$ and
$u_\ell,v_\ell,w_\ell\in\R$. Let $\Phi_\ell$ be given by
\eqref{eq:decoded-value} with the label $k_\ell$ in \eqref{eq:scalar-label},
and put $\Gamma_0=\Phi_0$ and $\Gamma_\ell=\Phi_\ell-\Phi_{\ell-1}$
for $\ell\ge1$. 
There exist  affine maps
$\bm{\calA}_r\in\aff_{\le 3dN+2d+4}$ for $r\ge1$ such that,
on $[0,1]^d$, both $\Phi_\ell$
and $\Gamma_\ell$ are scalar affine functions of
\begin{equation*}
\ocirc_{r=1}^{\ell+4}\bigl(\bmvarrho\circ\bm{\calA}_r\bigr),
\qquad \ell\ge0,
\end{equation*}
where $\bmvarrho$ is the mixed activation in \eqref{eq:mixed-activation}.
\end{proposition}
\end{colorenv}

We prove Proposition~\ref{lem:pipeline} in Section~\ref{sec:exact-phase}
by realizing both $\Phi_\ell$ and $\Gamma_\ell$ as affine readouts of
layer $\ell+4$. Layer $\ell+2$ provides the label $k_\ell$, and layer
$\ell+3$ evaluates the inner sine. Scaling the outer sine output in
layer $\ell+4$ gives $\Phi_\ell$. For $\ell\ge1$, subtracting the stored
value of $\Phi_{\ell-1}$ gives $\Gamma_\ell$. These layers also prepare
labels for finer scales. With the parameters chosen in
Proposition~\ref{prop:decoded-averages}, $\Phi_\ell$ satisfies
\eqref{eq:decoded-estimate}.

\subsection{Completion of the proof}\label{sec:infinite-compatibility}

We choose the decoder parameters at each scale and then apply Proposition~\ref{lem:pipeline} to realize all approximations and their successive differences along one sequence of hidden layers.
We first consider the case where $f$ equals a constant $c$ almost
everywhere. Set all hidden affine maps to zero, take the first correction
head to be the constant function $c$, and set all subsequent heads to
zero. Then $\Phi_\ell\equiv c$ for every $\ell\ge0$, and the constant
function $c$ is an affine readout of every hidden state. We henceforth
assume that $f$ is not almost everywhere constant.


\mystep{1}{Approximation at each scale.}
To apply Proposition~\ref{prop:decoded-averages}, we need to show that
$\eta_\ell \coloneqq \omega_{f,p}(N^{-\ell})>0$ for every $\ell\ge0$.
Since $f$ is not almost everywhere constant, $\|f-\bar f_0\|_{L^p}>0$,
so Proposition~\ref{prop:f-aver:p:norm:upper_bound} at level zero gives
$\omega_{f,p}(1)>0$. Fix $\ell\ge0$ and, for $\|\bmh\|_2\le1$, set
$\bmg=N^{-\ell}\bmh$. By convexity of the cube,
$E_{\bmh}+j\bmg\subseteq E_{\bmg}$ for $j=0,\ldots,N^\ell-1$.
Telescoping in steps of $\bmg$ and translating each term gives
\[
\|f(\cdot+\bmh)-f\|_{L^p(E_{\bmh})}
\le \sum_{j=0}^{N^\ell-1}
   \|f(\cdot+\bmg)-f\|_{L^p(E_{\bmh}+j\bmg)}
\le N^\ell\eta_\ell.
\]
Taking the supremum over $\|\bmh\|_2\le1$ yields
$\eta_\ell\ge N^{-\ell}\omega_{f,p}(1)>0$.

Next, we can apply Proposition~\ref{prop:decoded-averages} at each level with $\eta_\ell=\omega_{f,p}(N^{-\ell})$.
\blue{For each $\ell$, fix one admissible tuple $(\delta_\ell,u_\ell,v_\ell,w_\ell)$ using the full-cell averages and the tolerance at that level. This gives a single sequence of decoder parameters without prescribing a terminal depth.}
By combining \eqref{eq:projection-error} and \eqref{eq:decoded-estimate}, we have
\begin{equation*}
\begin{aligned}
\|f-\Phi_\ell\|_{L^p}
&\le \|f-\bar f_\ell\|_{L^p}+\|\bar f_\ell-\Phi_\ell\|_{L^p}\\
&\le (2^{1/p}d+1)\omega_{f,p}(N^{-\ell})
\le (2d+1)\omega_{f,p}(N^{-\ell}).
\end{aligned}
\end{equation*}
We remark that the estimate uses only the averaging and decoding errors at level $\ell$.

\mystep{2}{Realization of the multigrade sum.}
It remains to realize each $\Phi_\ell$ as a multigrade sum.
We can apply Proposition~\ref{lem:pipeline} to the parameters chosen in Step~1 and 
write $(\bm{\calA}_r)_{r\ge1}$ for its hidden affine maps.
The input dimension is $d$, and every hidden state has width $3dN+2d+4$, so these maps belong to $\aff_{\le 3dN+2d+4}$. At layer $j+4$, choose $\calA_j^{\mathrm{out}}$ to be an affine head that gives $\Gamma_j$ from $\bm{\mathcal V}_{j+4}$. At level zero this head already returns $\Phi_0$. For all $\ell\ge1$, summing the heads gives
\begin{equation*}
\sum_{j=0}^\ell \calA_j^{\mathrm{out}}\circ\bm{\mathcal V}_{j+4}
=\Phi_0+\sum_{j=1}^\ell(\Phi_j-\Phi_{j-1})
=\Phi_\ell .
\end{equation*}
Each $\Gamma_j$ is exactly the difference of consecutive decoded functions for $j\ge1$, so the sum equals the function estimated in Step~1 without accumulating decoding errors from earlier levels. Proposition~\ref{lem:pipeline} also gives $\Phi_\ell$ directly as an affine readout of $\bm{\mathcal V}_{\ell+4}$.

For any $0\le m<n$, the readout at level $n$ uses the same first $m+4$ affine maps and the same correction heads through level $m$. This includes the coefficients of neurons that already prepare finer labels in those layers, since the entire construction uses the parameter sequence fixed in Step~1. The bound therefore holds for every finite prefix, and $\omega_{f,p}(N^{-\ell})\to0$ gives $\Phi_\ell\to f$ in $L^p([0,1]^d)$.

\section{Proofs of the auxiliary propositions}\label{sec:proofs-propositions}


This section proves the three propositions used in the proof of
Theorem~\ref{thm:main:Lp}. The first two proofs concern a single grid. Section~\ref{sec:proof-cell-averages} bounds the error of replacing $f$ by its cell averages, and Section~\ref{sec:levelwise-data} shows how two sine activations approximate these averages. Section~\ref{sec:exact-phase} then connects the scales by constructing one sequence of affine maps that gives all the readouts in Theorem~\ref{thm:main:Lp}.

\subsection{Approximation by cell averages}\label{sec:proof-cell-averages}

To prove Proposition~\ref{prop:f-aver:p:norm:upper_bound}, we first
estimate the $L^p$ deviation of a function from its mean on an interval
$\calI$ of length $s$.
For each $t\in\calI$ where $g(t)$ is defined, the difference between
$g(t)$ and the mean of $g$ is the average of $g(t)-g(y)$ over $y\in\calI$.
Writing $h=y-t$ gives $|h|\le s$, and both $t$ and $t+h$ lie in $\calI$.
The following lemma uses this observation to bound the $L^p$ deviation
from the mean in terms of the translation differences appearing
in~\eqref{eq:modulus}.

\begin{lemma}\label{lem:f-aver:p:norm:upper_bound:1d}
Let $1\le p<\infty$ and let $g\in L^p(\calI)$ on an interval $\calI$ of length $s>0$. Write $\mathsf{A}_\calI(g)=s^{-1}\int_\calI g(y)dy$ for its mean. Then
\begin{equation*}
\int_\calI|g(t)-\mathsf{A}_\calI(g)|^pdt
\le\frac1s\int_{-s}^s\int_{\calI\cap(\calI-h)}
|g(t+h)-g(t)|^pdtdh.
\end{equation*}
\end{lemma}

\begin{proof}
By translation, we may assume that $\calI=[0,s]$ up to its endpoints.
Since $\calI$ has finite measure and $g\in L^p(\calI)$, we have
$g\in L^1(\calI)$, so $\mathsf{A}_{\calI}(g)$ is well defined.
For almost every $t\in\calI$,
\[
g(t)-\mathsf{A}_{\calI}(g)
=\frac{1}{s}\int_0^s \bigl(g(t)-g(y)\bigr)\,dy.
\]
By Jensen's inequality,
\[
|g(t)-\mathsf{A}_{\calI}(g)|^p
\le \frac{1}{s}\int_0^s |g(t)-g(y)|^p\,dy.
\]
Integrating in $t$ gives
\[
\int_0^s |g(t)-\mathsf{A}_{\calI}(g)|^p\,dt
\le \frac{1}{s}\int_0^s\int_0^s |g(t)-g(y)|^p\,dy\,dt.
\]

Extend $g$ by zero outside $\calI$. For almost every $t\in\calI$,
the change of variables $h=y-t$ yields
\[
\int_0^s |g(t)-g(y)|^p\,dy
=\int_{-t}^{s-t}|g(t+h)-g(t)|^p\,dh
=\int_{-s}^s \one_{\{t+h\in\calI\}}|g(t+h)-g(t)|^p\,dh.
\]
Tonelli's theorem then gives
\[
\begin{aligned}
\int_0^s\int_0^s |g(t)-g(y)|^p\,dy\,dt
&=\int_{-s}^s\int_0^s
  \one_{\{t+h\in\calI\}}|g(t+h)-g(t)|^p\,dt\,dh\\
&=\int_{-s}^s\int_{\calI\cap(\calI-h)}
  |g(t+h)-g(t)|^p\,dt\,dh.
\end{aligned}
\]
Substituting this identity into the previous estimate proves the lemma.
\end{proof}

With Lemma~\ref{lem:f-aver:p:norm:upper_bound:1d} in hand, we can now
prove Proposition~\ref{prop:f-aver:p:norm:upper_bound}.

\begin{proof}[Proof of Proposition~\ref{prop:f-aver:p:norm:upper_bound}]
Fix $\ell\ge0$ and put $s=N^{-\ell}$. Averaging one coordinate at a time
allows us to apply Lemma~\ref{lem:f-aver:p:norm:upper_bound:1d} using
only displacements of length at most $s$, even though a grid cell has
diameter $\sqrt d\,s$.
We first estimate the error on a single full cell
$Q=\widehat Q_{\ell,\bmbeta}$.
We then sum the resulting integrals
before applying the translation modulus.  We remark that this order matters because
the integrals for a fixed displacement occupy disjoint parts of the
cube. Bounding them separately by the same global supremum would
introduce a factor equal to the number of cells.

\mystep{1}{Estimate the error on one cell.}
For $\bmbeta=(\beta_1,\ldots,\beta_d)$, we write
\[
Q=\widetilde{\calI}_1\times\cdots\times\widetilde{\calI}_d,
\qquad
\tn{where}\  
\widetilde{\calI}_j:=\hatcalI_{\ell,\beta_j},
\qquad j=1,\ldots,d.
\]
Each interval $\widetilde{\calI}_j$ has length $s$.
For $1\le j\le d$, let $P_j$ be the operator that averages in the $j$th coordinate,
\begin{equation*}
(P_jg)(\bmx) \coloneqq \frac1s\int_{\tildecalI_j}
g(x_1,\ldots,x_{j-1},t,x_{j+1},\ldots,x_d)dt.
\end{equation*}
To see that $P_j$ is a contraction on $L^p(Q)$, fix all coordinates other than $x_j$. By Jensen's inequality, $|(P_jg)(\bmx)|^p$ is at most the average of $|g|^p$ on that slice. The left-hand side is constant in $x_j$. 
Integrating over \(\tildecalI_j\) introduces a factor \(s\) on the right, which cancels the factor \(1/s\).
Fubini's theorem then gives $\|P_jg\|_{L^p(Q)}\le\|g\|_{L^p(Q)}$.

Let $\widetilde P_0=\mathtt{Id}$ be the identity on $L^p(Q)$, and set
\begin{equation*}
\widetilde P_j=P_1\circ\cdots\circ P_j,\qquad 1\le j\le d.
\end{equation*}
Each $\widetilde P_j$ is a contraction. Fubini's theorem identifies the successive averages in all coordinates with the full-cell mean, so $\widetilde P_d f=\mathsf{A}_Q(f)$ almost everywhere on $Q$. For example, in two dimensions
\begin{equation*}
f-P_1P_2f=(f-P_1f)+P_1(f-P_2f).
\end{equation*}
The first term compares $f$ with its average in the first coordinate.
The second is an average of the analogous error in the other coordinate,
so its norm is at most $\|f-P_2f\|_{L^p(Q)}$. In general, insert all
partial averages between $f$ and $\mathsf{A}_Q(f)$. Since $\widetilde P_j=\widetilde P_{j-1}\circ P_j$, their successive differences satisfy
\begin{equation*}
f-\mathsf{A}_Q(f)=\widetilde P_0 f-\widetilde P_d f
=\sum_{j=1}^d \widetilde P_{j-1}(\mathtt{Id}-P_j)f .
\end{equation*}
Take the norm and use the triangle inequality. To raise the resulting
sum to the $p$th power, apply convexity to the arithmetic mean of its
$d$ nonnegative terms. This gives
$(\sum_{j=1}^d t_j)^p\le d^{p-1}\sum_{j=1}^d t_j^p$ and hence
\begin{equation*}
\begin{aligned}
\|f-\mathsf{A}_Q(f)\|_{L^p(Q)}^p
&\le d^{p-1}\sum_{j=1}^d
\bigl\|\widetilde P_{j-1}(\mathtt{Id}-P_j)f\bigr\|_{L^p(Q)}^p\\
&\le d^{p-1}\sum_{j=1}^d \|f-P_jf\|_{L^p(Q)}^p .
\end{aligned}
\end{equation*}
The last step uses the contractivity of each $\widetilde P_{j-1}$.

For the $j$th term, fix the other $d-1$ coordinates and apply Lemma~\ref{lem:f-aver:p:norm:upper_bound:1d} on the interval $\tildecalI_j$. Almost every slice belongs to $L^p(\tildecalI_j)$ by Fubini's theorem. After integration over the remaining coordinates, the points in the inner integral are exactly those $\bmx\in Q$ for which $\bmx+h\bm e_j\in Q$, where $\bm e_j$ is the $j$th coordinate vector. We obtain
\begin{equation*}
\|f-P_jf\|_{L^p(Q)}^p
\le\frac1s\int_{-s}^s\int_{Q\cap(Q-h\bm e_j)}
|f(\bmx+h\bm e_j)-f(\bmx)|^p d\bmx dh.
\end{equation*}
Substitution into the preceding estimate gives
\begin{equation}\label{eq:local-average-bound}
\|f-\mathsf{A}_Q(f)\|_{L^p(Q)}^p
\le\frac{d^{p-1}}s\sum_{j=1}^d\int_{-s}^s
\int_{Q\cap(Q-h\bm e_j)}|f(\bmx+h\bm e_j)-f(\bmx)|^p d\bmx dh.
\end{equation}

\mystep{2}{Sum over the grid.}
Since $\bar f_\ell$ equals $\mathsf{A}_Q(f)$ on each cell, summing the left-hand side of \eqref{eq:local-average-bound} gives $\|f-\bar f_\ell\|_{L^p}^p$. Fix $j$ and $h$, and take the translation vector $\bmh=h\bm e_j$. As $Q$ ranges over the full cells, the sets $Q\cap(Q-\bmh)$ are disjoint up to their boundaries. Their union contains precisely the points for which $\bmx$ and $\bmx+\bmh$ lie in the same cell, up to boundary sets. Thus it is contained in $E_{\bmh}=[0,1]^d\cap([0,1]^d-\bmh)$ from \eqref{eq:modulus}. We may therefore bound their combined integrals by the integral on $E_{\bmh}$. This adds only nonnegative terms, from
pairs in different cells. Since $\|\bmh\|_2=|h|\le s$, the definition of the modulus yields
\begin{equation*}
\begin{aligned}
\|f-\bar f_\ell\|_{L^p}^p
&\le\frac{d^{p-1}}s\sum_{j=1}^d\int_{-s}^s
\|f(\cdot+h\bm e_j)-f(\cdot)\|_{L^p(E_{h\bm e_j})}^p dh\\
&\le\frac{d^{p-1}}s\bigl(d\cdot 2s\bigr)\omega_{f,p}(s)^p
=2d^p\,\omega_{f,p}(s)^p .
\end{aligned}
\end{equation*}
The interval of possible scalar displacements has length $2s$.
Dividing by the averaging length $s$ leaves the factor $2$, and
summing over the $d$ coordinates contributes the last factor $d$.
Taking $p$th roots gives $2^{1/p}d$ in \eqref{eq:projection-error}.
\end{proof}

\subsection{Decoding the cell averages}\label{sec:levelwise-data}

To prove Proposition~\ref{prop:decoded-averages},
we introduce Lemma~\ref{prop:k:to:yk}, which allows the decoder to
approximate the prescribed cell average at each label. This result was
proved in~\cite[Proposition~A.2]{ZZZZ-25-FMMNN}. 
For completeness, we give an alternative proof that highlights the
different roles of the two frequencies in the precision analysis
of Section~\ref{sec:precision-capacity}.

\begin{lemma}\label{prop:k:to:yk}
Let $K\ge1$ be an integer and prescribe a real value $y_k$ at each integer $k=1,\ldots,K$. Set $u=\max_{1\le k\le K}|y_k|$. For any $\epsilon>0$ there are two frequencies $v,w\in\R$ such that
\begin{equation*}
\bigl|u\sin\bigl(v\sin(wk)\bigr)-y_k\bigr|<\epsilon,
\qquad k=1,\ldots,K.
\end{equation*}
\end{lemma}

\begin{proof}
When $u=0$, all the prescribed values vanish and we can take $v=w=0$. Suppose $u>0$. We will choose the inner frequency $w$ using only the number $K$ of labels. With this $w$ fixed, varying the outer frequency $v$ will allow us to fit all the prescribed values at once. To specify what the phases must approximate, put $\xi_k=\arcsin(y_k/u)$, so $u\sin(\xi_k)=y_k$. It suffices to arrange that $v\sin(wk)$ is close to $\xi_k$ modulo $2\pi$ for every $k$. Indeed, if there are integers $n_1,\ldots,n_K$ satisfying
\begin{equation*}
\bigl|v\sin(wk)-\xi_k-2\pi n_k\bigr|<\epsilon/u,
\qquad k=1,\ldots,K,
\end{equation*}
then periodicity and the $1$-Lipschitz continuity of sine give
\begin{equation*}
\bigl|u\sin\bigl(v\sin(wk)\bigr)-y_k\bigr|
\le u\bigl|v\sin(wk)-(\xi_k+2\pi n_k)\bigr|<\epsilon .
\end{equation*}

First choose $w$ so that the $K+1$ numbers
\begin{equation*}
1,\qquad \frac{\sin(w)}{2\pi},\qquad\ldots,\qquad\frac{\sin(Kw)}{2\pi}
\end{equation*}
are rationally independent. This means that no integer combination
of them can vanish unless all its coefficients are zero. 
To see why this condition is necessary for the integer multiples to be dense in the torus, consider the \(K\) phase coordinates modulo one.
If the
coordinates and $1$ satisfied a nontrivial integer relation, then
some nonzero integer combination of the coordinates would be an
integer. The same combination would still be an integer after
multiplication by any integer $v$. This restriction would prevent
those phases from approximating arbitrary prescribed phases.
To construct an admissible $w$, fix a nonzero integer vector $\bm c=(c_0,\ldots,c_K)$ and consider the trigonometric polynomial
\begin{equation*}
2\pi c_0+\sum_{k=1}^K c_k\sin(kt).
\end{equation*}
If this expression vanished identically, integration over a period would give $c_0=0$. Multiplication by $\sin(jt)$ and integration over the same period would then give $\pi c_j=0$ for $1\le j\le K$, a contradiction. It has only finitely many zeros on $[0,2\pi)$. Indeed, with
$z=e^{it}$ we have $\sin(kt)=(z^k-z^{-k})/(2i)$. Multiplying
the expression by $z^K$ gives a nonzero polynomial of degree at
most $2K$, which has at most $2K$ roots. 
There are countably many nonzero integer vectors \(\bmc\), so the union of the zero sets of the corresponding trigonometric polynomials is countable.
Choosing $w$ outside that union rules out every nontrivial integer relation among the indicated numbers.

For this choice of $w$, Kronecker's approximation theorem says that the vectors
\begin{equation*}
\frac{v}{2\pi}\bigl(\sin(w),\ldots,\sin(Kw)\bigr),\qquad v\in\mathbb{Z},
\end{equation*}
are dense in the $K$-dimensional torus. Geometrically
each coordinate is a point on a circle of circumference one, so the
whole vector is a point on the product of $K$ circles. Prescribe
a small arc around each target phase. Density means that one and
the same integer $v$ brings every coordinate into its own arc.
Writing the circle coordinates as real numbers introduces an integer
offset $n_k$ for each $k$. Take the target coordinates to be $\xi_k/(2\pi)$ and the accuracy in each coordinate to be $\epsilon/(2\pi u)$. We obtain one integer $v$ and integers $n_1,\ldots,n_K$ such that
\begin{equation*}
\biggl|\frac{v\sin(wk)}{2\pi}-\frac{\xi_k}{2\pi}-n_k\biggr|
<\frac{\epsilon}{2\pi u},
\qquad 1\le k\le K.
\end{equation*}
Multiplying by $2\pi$ gives the required phase inequalities. This use of density supplies a common outer frequency for all $K$ values, with no estimate on how large it must be.
\end{proof}

With Lemma~\ref{prop:k:to:yk} in hand, we  are ready to prove
Proposition~\ref{prop:decoded-averages}.

\begin{proof}[Proof of Proposition~\ref{prop:decoded-averages}]
Fix $\ell\ge0$, put $s=N^{-\ell}$, and prescribe a tolerance
$\eta_\ell>0$. Figure~\ref{fig:parameter-order} gives the order of
our choices. We allow an error of $\eta_\ell/2$ in the norm on each
of two disjoint regions. On $U_\ell$ we fit finitely many numbers
with a pointwise error of at most $\eta_\ell/2$. On $\Omega_\ell$
we control only the magnitude of the decoder, and we choose the
measure of this region small enough to give the other half of the
tolerance. These choices are compatible because the target numbers
are averages over full cells, determined before the boxes are shrunk.

\begin{figure}[htbp]
\centering
\includegraphics[width=0.9585\textwidth]{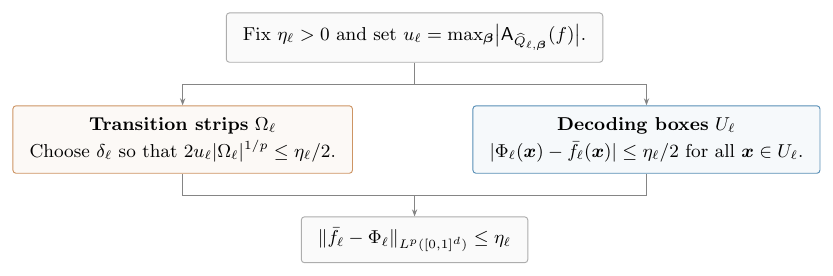}
\caption{Parameter choices at level $\ell$. The full-cell averages determine
$u_\ell$. We choose $\delta_\ell$ to make the $L^p$ error on $\Omega_\ell$
at most $\eta_\ell/2$, and choose the frequencies so that $\Phi_\ell$
approximates $\bar f_\ell$ within $\eta_\ell/2$ at every point of $U_\ell$.
Together these bounds give an $L^p$ error of at most $\eta_\ell$ on $[0,1]^d$.}
\label{fig:parameter-order}
\end{figure}

\mystep{1}{Fix the cell values and identify their labels.}
Every full cell $\widehat Q_{\ell,\bmbeta}$ has measure $s^d$, so H\"older's inequality gives
\begin{equation*}
\bigl|\mathsf{A}_{\widehat Q_{\ell,\bmbeta}}(f)\bigr|
\le |\widehat Q_{\ell,\bmbeta}|^{-1/p}\|f\|_{L^p(\widehat Q_{\ell,\bmbeta})}
=N^{d\ell/p}\|f\|_{L^p(\widehat Q_{\ell,\bmbeta})}.
\end{equation*}
Taking the maximum over the finitely many cells gives a finite amplitude $u_\ell=\max_{\bmbeta}|\mathsf{A}_{\widehat Q_{\ell,\bmbeta}}(f)|$, even when $f$ is unbounded. This maximum may grow as the grid is
refined. Finiteness at the current level suffices, because we choose
a fresh strip width for each level. If $u_\ell=0$, $\bar f_\ell$ is zero almost everywhere and the decoder with this amplitude is identically zero. The conclusion holds with any admissible $\delta_\ell$ and $v_\ell=w_\ell=0$.

Assume now that $u_\ell>0$. To apply
Lemma~\ref{prop:k:to:yk}, form a list of $N^{d\ell}$ numbers.
For each cell index $\bmbeta$, assign its full-cell average to the entry
\begin{equation*}
k=1+\sum_{j=1}^d(N^\ell)^{j-1}\beta_j,
\qquad y_k=\mathsf{A}_{\widehat Q_{\ell,\bmbeta}}(f).
\end{equation*}
As shown after \eqref{eq:scalar-label}, each index occurs once in this
list, including the single entry $k=1$ at level zero. Moreover,
$k_\ell(\bmx)=k$ on the corresponding interior box because
$\psi_\ell(x_j)=\beta_js$ in every coordinate. Fitting $y_k$ at its
integer label thus fits the required average throughout that box.

\mystep{2}{Control the errors on the two parts of the cube.}
The set $U_\ell$ is the union of the smaller boxes $Q_{\ell,\bmbeta}$, and $\Omega_\ell$ is its complement in the unit cube. Each smaller box has measure $s^d(1-\delta_\ell)^d$ and there are $s^{-d}$ boxes. They have disjoint interiors. Consequently,
\begin{equation*}
|\Omega_\ell|=1-(1-\delta_\ell)^d\le d\delta_\ell .
\end{equation*}
The inequality follows from $(1-\delta_\ell)^d\ge1-d\delta_\ell$. The set $\Omega_\ell$ includes the outer faces $x_j=1$, which do not affect its measure. Since $u_\ell$ has already been fixed, any choice
\begin{equation*}
0<\delta_\ell<\min\biggl\{\frac13,\frac1d\biggl(\frac{\eta_\ell}{4u_\ell}\biggr)^p\biggr\}
\end{equation*}
makes $2u_\ell|\Omega_\ell|^{1/p}\le\eta_\ell/2$. Indeed,
$|\Omega_\ell|^{1/p}\le(d\delta_\ell)^{1/p}$ and the choice above
bounds this last quantity by $\eta_\ell/(4u_\ell)$.

Apply Lemma~\ref{prop:k:to:yk} to the $N^{d\ell}$ full-cell averages arranged in label order, with accuracy $\eta_\ell/2$. 
On \(Q_{\ell,\bmbeta}\), the decoder input is the integer label assigned to the corresponding full-cell average.
Therefore $|\bar f_\ell-\Phi_\ell|\le\eta_\ell/2$ at every point of $U_\ell$. On $\Omega_\ell$ the integer interpolation does not apply, but $|\bar f_\ell|\le u_\ell$ almost everywhere and $|\Phi_\ell|\le u_\ell$ everywhere since $|\sin|\le1$. Thus the error on this part has $L^p$-norm at most $2u_\ell|\Omega_\ell|^{1/p}$. Combining the bounds gives
\begin{equation*}
\begin{aligned}
\|\bar f_\ell-\Phi_\ell\|_{L^p}
&\le\|\bar f_\ell-\Phi_\ell\|_{L^p(U_\ell)}
  +\|\bar f_\ell-\Phi_\ell\|_{L^p(\Omega_\ell)}\\
&\le (\eta_\ell/2)|U_\ell|^{1/p}
  +2u_\ell|\Omega_\ell|^{1/p}
\le \eta_\ell.
\end{aligned}
\end{equation*}
The last inequality uses $|U_\ell|\le1$. \blue{Both the strip width and the two frequencies can therefore be chosen using only the full-cell averages and the tolerance at level $\ell$.}
\end{proof}

\subsection{Realization in one network}\label{sec:exact-phase}\label{sec:phase-rounding}

To prove Proposition~\ref{lem:pipeline}, we construct a single sequence
of affine maps that realizes the approximations at all levels. Each layer
that supplies a readout must retain enough information to compute the
next one. To see what this requires, consider a single input coordinate
$x\in[0,1]$. Write $N^\ell x=m+t$, where $m$ is the cell index and
$t\in[0,1]$ is the position of $x$ within the cell. On the decoding
interval $\calI_{\ell,m}=[ms,(m+1-\delta_\ell)s]$ defined in Section~\ref{sec:auxiliary-statements}, the current readout needs only the rounded
value $mN^{-\ell}$. At the next level, this cell splits into $N$
children. For $0\le t<1$, the child containing $x$ has index
$Nm+\lfloor Nt\rfloor$. Thus two points with the same rounded value
at one level may require different rounded values at the next.



We retain $x$ and encode its position $t$ within the current cell by a
continuous phase. Since $t$ jumps from $1$ to $0$ across each interior
cell boundary, we let $\bm z(t)$ trace the boundary of a triangle once
as $t$ runs from $0$ to $1$. 
Take the vertices in $\R^3$ to be
\[
\bm p_0=(1,0,0),\qquad
\bm p_1=(0,1,0),\qquad
\bm p_2=(0,0,1).
\]
We index the vertices $\bm p_i$ and the coordinates $y_i$ of
$\bm y=(y_0,y_1,y_2)$ cyclically modulo three, so that
$\bm p_3=\bm p_0$ and $y_3=y_0$.
Define the $1$-periodic map $\bm z:\R\to\R^3$ by linear interpolation
along successive edges, namely,
\begin{equation}\label{eq:triangle-phase}
\bm z\Bigl(\frac{i+\theta}{3}\Bigr)
\coloneqq (1-\theta)\bm p_i+\theta\bm p_{i+1},
\qquad i=0,1,2,\qquad 0\le\theta\le1.
\end{equation}
In coordinates, for $t\in[0,1]$ we have
\[
\bm z(t)=
\begin{cases}
(1-3t,\,3t,\,0),   & 0\le t\le 1/3,\\
(0,\,2-3t,\,3t-1), & 1/3\le t\le 2/3,\\
(3t-2,\,0,\,3-3t), & 2/3\le t\le 1.
\end{cases}
\]
Since $\bm z(0)=\bm z(1)$, the two representations of a grid point
give the same phase, while distinct positions in $[0,1)$ have distinct
phases. 
For each $x\in[0,1]$, define its phase at level $\ell$ by
$\bm z_\ell(x)\coloneqq\bm z(N^\ell x)$. Writing $N^\ell x=m+t$
and using the periodicity of $\bm z$, we have
\[
\bm z_\ell(x)=\bmz (N^\ell x)=\bm z(m+t)=\bm z(t),
\qquad
\bm z_{\ell+1}(x)=\bmz (N^{\ell+1} x)=\bm z(Nm+Nt)=\bm z(Nt).
\]
Thus the same update $\bm z(t)\mapsto\bm z(Nt)$ applies in every cell
at every level. As illustrated in Figure~\ref{fig:triangle-phase} for
$N=2$, the phase varies linearly along each edge, which we use to
realize the update exactly with \ReLU{} neurons. For rounding, we also
use the retained input $x$, since points in different cells can have
the same phase. On a decoding interval, subtracting the within-cell
offset $N^{-\ell}t$ from $x$ gives the left endpoint $mN^{-\ell}$.

\begin{figure}[!htbp]
\centering
\includegraphics[width=0.94\textwidth]{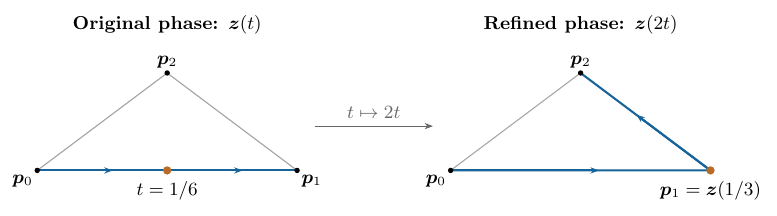}
\caption{Phase refinement for $N=2$. As $t$ runs from $0$ to $1/3$,
$\bm z(t)$ traverses the first edge and $\bm z(2t)$ traverses the
first two edges. The orange points show how refinement maps the
midpoint $\bm z(1/6)$ to the vertex $\bm z(1/3)=\bm p_1$.}
\label{fig:triangle-phase}
\end{figure}


\begin{colorenv}
The following two lemmas formalize these operations.
Lemma~\ref{lem:phase} constructs the exact update from $\bm z(t)$
to $\bm z(Nt)$ on the boundary of the triangle.
Lemma~\ref{lem:rounding} uses $\bm z(t)$ together with $x$ to recover
the rounded coordinate at the current scale. Since the phase update
depends only on $\bm z(t)$, the transition strip width can be chosen
separately at each scale.
\end{colorenv}

\begin{colorenv}
\begin{lemma}\label{lem:phase}
Let $N\ge2$ be an integer and let $\bm z$ be the periodic map
in \eqref{eq:triangle-phase}. There exist
$\bm{\calA}^{\mathrm{phase}}_N\in\aff(3,3N)$ and
$\bmW_N^{\mathrm{phase}}\in\R^{3\times 3N}$ such that the map
$\bm F_N=\ReLU\circ\bm{\calA}^{\mathrm{phase}}_N$ has the following properties for every $t\in\R$.
\begin{itemize}
\item The first three coordinates of $\bm F_N(\bm z(t))\in \R^{3N}$ form the
      current phase $\bm z(t)$.
\item The refined phase $\bm z(Nt)$ is given by the linear readout
      \begin{equation}\label{eq:phase-readouts}
      \bmW_N^{\mathrm{phase}}\bm F_N(\bm z(t))=\bm z(Nt).
      \end{equation}
\end{itemize}
\end{lemma}
\end{colorenv}
\begin{proof}
Our goal is to construct a map $\bm R_N\colon\R^3\to\R^3$
that recovers the refined phase from the input phase, that is,
\[
\bm R_N(\bm z(t))=\bm z(Nt)
\qquad\text{for every }t\in\R.
\]
On the triangle boundary, we must realize $\bm R_N$ by an affine map
followed by $3N$ \ReLU{} neurons and a linear readout. The first three
neurons must also retain the input phase $\bm z(t)$. We proceed in
three steps. First, we determine the vertex values and slope jumps of
the refined phase. Second, we construct $\bm R_N$ by matching the
vertex values and adding a correction at each interior knot. Third,
we realize this construction with the prescribed neurons and readout.

\mystep{1}{Determine the values and slopes on each edge.}
Let edge $i$ run from $\bm p_i$ to $\bm p_{i+1}$, where $i=0,1,2$
and vertex indices are read modulo three. On this edge, write
$t=(i+\theta)/3$ with $0\le\theta\le1$. By
\eqref{eq:triangle-phase}, the input phase is
\[
\bm z(t)=(1-\theta)\bm p_i+\theta\bm p_{i+1}.
\]
Throughout the proof, the slope of a vector-valued path on an affine
piece means its derivative vector with respect to $\theta$.

As $\theta$ increases by $1/N$, the argument $Nt$ increases by $1/3$.
Thus, between two consecutive subdivision points, the refined phase
$\bm z(Nt)$ traverses one edge of the triangle. Its value at the
$r$th subdivision point $\theta=r/N$ is
\[
\bm v_{i,r}\coloneqq
\bm z\Bigl(\frac{Ni+r}{3}\Bigr)=\bm p_{Ni+r},
\qquad r=0,\ldots,N.
\]
For example, when $N=2$, the refined phase on the first input edge
takes the values $\bm p_0$, $\bm p_1$, and $\bm p_2$ at the left
endpoint, midpoint, and right endpoint, respectively
(Figure~\ref{fig:triangle-phase}).

On the subinterval $[r/N,(r+1)/N]$, the refined phase varies linearly
from $\bm v_{i,r}$ to $\bm v_{i,r+1}$. Since this subinterval has
length $1/N$, its slope is $N(\bm v_{i,r+1}-\bm v_{i,r})$. At an
interior knot $\theta=r/N$, subtracting the slope on the left from
the slope on the right gives the slope jump
\[
\bm c_{i,r}\coloneqq
N\bigl(\bm v_{i,r+1}-2\bm v_{i,r}+\bm v_{i,r-1}\bigr),
\qquad r=1,\ldots,N-1.
\]

Adjacent edges prescribe the same value at their common vertex, since
\[
\bm v_{i,N}=\bm v_{i+1,0}\qquad\text{for }i=0,1,
\qquad
\bm v_{2,N}=\bm v_{0,0}.
\]
We can therefore match all three vertex values with the linear map
$\bm{\calA}_N^{\mathrm{vert}}\colon\R^3\to\R^3$ given by
\[
\bm{\calA}_N^{\mathrm{vert}}(\bmy)
\coloneqq \sum_{i=0}^2 y_i\bm v_{i,0},
\qquad \bmy=(y_0,y_1,y_2).
\]
Indeed, this map sends $\bm p_i$ to $\bm v_{i,0}$, so on each edge
it interpolates linearly between the two prescribed endpoint values.
It remains to introduce the slope jumps $\bm c_{i,r}$ while
preserving those values.

\mystep{2}{Construct the correction at each interior knot.}
Fix edge $i$ and an interior knot $\theta=r/N$. The function
$\ReLU(\theta-r/N)$ has a unit slope jump at this knot. Its values
at $\theta=0$ and $\theta=1$ are $0$ and $1-r/N$, respectively.
Subtracting the line $(1-r/N)\theta$ makes the correction vanish at
both endpoints and preserves the slope jump.

To obtain this correction on edge $i$ and make it vanish on the other
two edges, define $J_{i,r}\colon\R^3\to\R$ by
\[
J_{i,r}(\bmy)
\coloneqq
\ReLU\bigl((1-r/N)y_{i+1}-(r/N)y_i\bigr)
-(1-r/N)y_{i+1},
\]
where $y_{i+1}=y_0$ when $i=2$. On edge $i$, the coordinates are
$y_i=1-\theta$ and $y_{i+1}=\theta$, so
\[
J_{i,r}(\bmy)
=\ReLU(\theta-r/N)-(1-r/N)\theta
=
\begin{cases}
-(1-r/N)\theta, & 0\le\theta\le r/N,\\
-(r/N)(1-\theta), & r/N\le\theta\le1.
\end{cases}
\]
Its slope increases from $-(1-r/N)$ to $r/N$ at $\theta=r/N$, a
jump of one. \blue{On each of the other two edges, at least one of
$y_i$ and $y_{i+1}$ is zero. If $y_i=0$, the \ReLU{} term equals
$(1-r/N)y_{i+1}$ and the two terms cancel. If $y_{i+1}=0$, both
terms vanish.} Thus the scalar function $J_{i,r}$ determines where
the correction acts. Multiplying it by $\bm c_{i,r}$ gives the
required vector slope jump at $\theta=r/N$ while preserving all
vertex values. Figure~\ref{fig:edge-hinge} illustrates $J_{i,r}$.

\begin{figure}[!htbp]
\centering
\includegraphics[width=0.99685\textwidth]{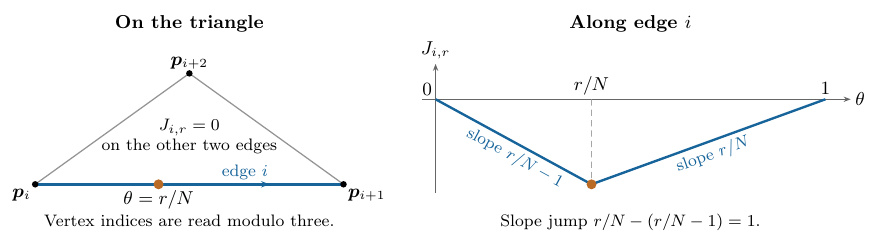}
\caption{Along edge $i$, from $\bm p_i$ to $\bm p_{i+1}$, the slope
of $J_{i,r}$ increases by one at $\theta=r/N$. The correction vanishes
at both endpoints and along the other two edges.}
\label{fig:edge-hinge}
\end{figure}

We define $R_N:\R^3\to \R^3$ by adding these corrections to the vertex map, namely,
\begin{equation}\label{eq:TN-explicit}
\bm R_N(\bmy)
=\bm{\calA}_N^{\mathrm{vert}}(\bmy)
+\sum_{i=0}^2\sum_{r=1}^{N-1}J_{i,r}(\bmy)\,\bm c_{i,r}.
\end{equation}
We now show that $\bm R_N$ gives the desired output at every point
on each edge. Fix edge $i$, write $t=(i+\theta)/3$, and consider
the difference between the constructed and desired paths,
\[
\bm D(\theta)\coloneqq \bm R_N(\bm z(t))-\bm z(Nt),
\qquad 0\le\theta\le1.
\]
Both paths are continuous and affine on each of the $N$ subintervals
\[
\left[\frac rN,\frac{r+1}{N}\right],
\qquad r=0,\ldots,N-1.
\]
Thus $\bm D$ has a constant slope on each subinterval. To compare
these slopes, consider the two subintervals meeting at an interior
knot $\theta=r/N$,
\[
\left[\frac{r-1}{N},\frac rN\right]
\qquad\text{and}\qquad
\left[\frac rN,\frac{r+1}{N}\right].
\]
By Step~1, the slope of the desired path changes by $\bm c_{i,r}$
from the left subinterval to the right. In the constructed path,
only the term $J_{i,r}\bm c_{i,r}$ contributes to the slope jump
at this knot, and its contribution is also $\bm c_{i,r}$.
The two jumps cancel in $\bm D$, so $\bm D$ has the same slope
on both subintervals. This holds at every interior knot. Hence
$\bm D$ has the same slope on all $N$ subintervals and, by
continuity, is affine on $[0,1]$.

At each endpoint, the vertex map gives the prescribed value and all
corrections vanish. Hence $\bm D(0)=\bm D(1)=\bm 0$. By affine
interpolation,
\[
\bm D(\theta)=(1-\theta)\bm D(0)+\theta\bm D(1)=\bm 0
\qquad (0\le\theta\le1).
\]
This proves $\bm R_N(\bm z(t))=\bm z(Nt)$ on edge $i$. Together,
the three edges give the identity for $t\in[0,1]$. Since $\bm z$
is $1$-periodic and $N$ is an integer, both sides are $1$-periodic
in $t$, so the identity holds for every $t\in\R$.

\mystep{3}{Realize the construction with $3N$ neurons.}
Use the first three neurons to compute $\ReLU(y_i)$ for $i=0,1,2$.
On the triangle boundary, all three coordinates are nonnegative, so
these outputs reproduce $\bmy$. For each edge $i$ and each interior
knot $r/N$, use one additional neuron to compute
\[
\ReLU\bigl((1-r/N)y_{i+1}-(r/N)y_i\bigr).
\]
This gives $3+3(N-1)=3N$ neurons. Order the additional neurons first
by $i$ and then by $r$. Their preactivations, together with those of
the first three neurons, define a single affine map
$\bm{\calA}_N^{\mathrm{phase}}\in\aff(3,3N)$ with zero bias.
Set $\bm F_N=\ReLU\circ\bm{\calA}_N^{\mathrm{phase}}$.

For $\bmy=\bm z(t)$, the first three entries of $\bm F_N(\bmy)$
reproduce the input phase, proving the first assertion. These entries
also supply the coordinates $y_i$ used in the vertex map. To recover
$J_{i,r}(\bmy)$, subtract $(1-r/N)y_{i+1}$ from the output of the
neuron indexed by $(i,r)$. Thus every coordinate of $\bm R_N(\bmy)$
in \eqref{eq:TN-explicit} is a linear combination of the entries of
$\bm F_N(\bmy)$. Taking the coefficients of these three combinations
as the rows of $\bmW_N^{\mathrm{phase}}\in\R^{3\times 3N}$ gives
\[
\bmW_N^{\mathrm{phase}}\bm F_N(\bm z(t))
=\bm R_N(\bm z(t))
=\bm z(Nt)
\qquad\text{for every }t\in\R.
\]
This proves \eqref{eq:phase-readouts}. Both the affine map and the
readout matrix depend only on $N$.
\end{proof}

Lemma~\ref{lem:phase} gives the phase at the next scale. We now use
the current phase $\bm z(N^\ell x)$ together with the retained input
$x$ to recover the rounded coordinate $\psi_\ell(x)$ defined in
\eqref{eq:continuous-rounding}. Set $s=N^{-\ell}$ and write
$x=(m+t)s$, where $0\le t\le1$ is the normalized position in cell $m$.
By periodicity, $\bm z(N^\ell x)=\bm z(m+t)=\bm z(t)$.

\begin{figure}[htbp]
\centering
\includegraphics[width=0.9485\textwidth]{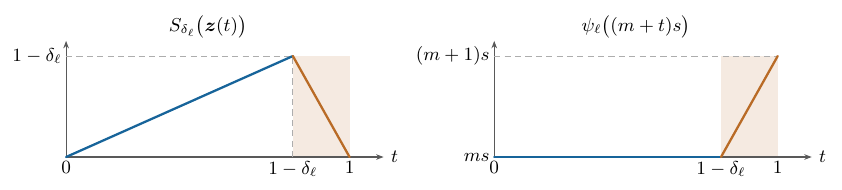}
\caption{Rounding on one grid interval, with $x=(m+t)s$ and
$s=N^{-\ell}$. The left panel shows the normalized offset
$S_{\delta_\ell}(\bm z(t))$. Subtracting $s$ times this offset from
$x$ gives the rounded coordinate in the right panel. The shaded strip
has width $\delta_\ell$ in $t$ and $s\delta_\ell$ in $x$.}
\label{fig:phase-rounding}
\end{figure}

The right panel of Figure~\ref{fig:phase-rounding} shows the rounded
coordinate $\psi_\ell((m+t)s)$ as a function of $t$, and the left panel
shows the corresponding normalized offset $(x-\psi_\ell(x))/s$.
On the interior interval $\calI_{\ell,m}=[ms,(m+1-\delta_\ell)s]$
defined in Section~\ref{sec:auxiliary-statements}, we have
$\psi_\ell(x)=ms$, so the offset equals $(x-ms)/s=t$.
Across the shaded strip, $\psi_\ell(x)$ increases linearly from $ms$
to $(m+1)s$, while the offset decreases linearly from $1-\delta_\ell$
to zero. We therefore seek a function $S_{\delta_\ell}$ of the phase
$\bm z(t)$ that reproduces the profile in the left panel of Figure~\ref{fig:phase-rounding}. Subtracting
$s$ times its value from the retained input $x$ then gives the rounded
coordinate in the right panel
\[
\psi_\ell((m+t)s)=\psi_\ell(x)=x-sS_{\delta_\ell}\bigl(\bm z(t)\bigr).
\]

For $0<\delta<1/3$, the desired offset equals $t$ up to $t=1-\delta$
and then decreases linearly to zero at $t=1$. Writing
$\bmy=\bm z(t)$, we see from \eqref{eq:triangle-phase} that
$(y_1+2y_2)/3=t$ on the first two edges. We therefore need a
correction only on the final edge, where adding $y_0$ to this
expression gives $t$.
On the final edge, $y_0$ increases linearly from $0$ to $1$, while
$y_2$ decreases linearly from $1$ to $0$. We use $y_0$ as the
correction up to the breakpoint and a multiple of $y_2$ thereafter,
so that the offset returns to zero at $t=1$. 
At the breakpoint $t=1-\delta$, we have
$(y_0,y_2)=(1-3\delta,3\delta)$, so scaling $y_2$ by
$(1-3\delta)/(3\delta)$ makes it equal to $y_0$ there.
Along the final edge, $y_0$ increases while the scaled $y_2$
decreases, so their minimum gives the required correction
\[
\min\biggl\{y_0,\frac{1-3\delta}{3\delta}y_2\biggr\}
=
\begin{cases}
y_0, & 2/3\le t\le 1-\delta,\\[2pt]
\dfrac{1-3\delta}{3\delta}y_2, & 1-\delta<t\le 1.
\end{cases}
\]
This minimum also vanishes on the first two edges, where either
$y_0=0$ or $y_2=0$. Adding it to $(y_1+2y_2)/3$ and writing the
minimum in \ReLU{} form gives
\begin{equation}\label{eq:Sdelta}
S_\delta(\bmy)
\coloneqq \frac{y_1+2y_2}{3}+y_0
-\ReLU\biggl(y_0-\frac{1-3\delta}{3\delta}y_2\biggr),
\qquad \bmy=(y_0,y_1,y_2)\in\R^3.
\end{equation}
The following lemma verifies that $S_\delta(\bm z(t))$ gives the
desired offset, shown for $\delta=\delta_\ell$ in the left panel of
Figure~\ref{fig:phase-rounding}.

\begin{colorenv}
\begin{lemma}\label{lem:rounding}
Let $\ell\ge0$ be an integer and $0<\delta_\ell<1/3$.
The phase $\bm z_\ell(x)=\bm z(N^\ell x)$ determines the normalized
offset of the rounding function $\psi_\ell$ in
\eqref{eq:continuous-rounding} through the identity
\begin{equation}\label{eq:phase-rounding-identity}
\psi_\ell(x)
=x-N^{-\ell}S_{\delta_\ell}\bigl(\bm z_\ell(x)\bigr),
\qquad x\in[0,1].
\end{equation}
\end{lemma}
\end{colorenv}

\begin{proof}
Set $s=N^{-\ell}$ and $\delta=\delta_\ell$. For $x\in[0,1]$, write
$x=(m+t)s$ with $m\in\{0,\ldots,N^\ell-1\}$ and $0\le t\le1$.
Since $N^\ell x=m+t$ and $\bm z$ is $1$-periodic,
$\bm z_\ell(x)=\bm z(t)$. By \eqref{eq:continuous-rounding},
the normalized offset is
\[
\frac{x-\psi_\ell(x)}{s}
=t-\frac{\ReLU(t-1+\delta)}{\delta}
=\min\biggl\{t,\frac{(1-\delta)(1-t)}{\delta}\biggr\}.
\]
It therefore suffices to show that $S_\delta(\bm z(t))$ equals
this minimum for $0\le t\le1$.

Write $\bmy=(y_0,y_1,y_2)=\bm z(t)$. Rewriting \eqref{eq:Sdelta}
in terms of a minimum gives
\[
S_\delta(\bmy)
=\frac{y_1+2y_2}{3}
 +\min\biggl\{y_0,\frac{1-3\delta}{3\delta}y_2\biggr\}.
\]
On the first edge $y_2=0$, and on the second $y_0=0$, so the
minimum vanishes on both edges. Hence \eqref{eq:triangle-phase}
gives $S_\delta(\bmy)=(y_1+2y_2)/3=t$, which is the required
offset for $0\le t\le2/3<1-\delta$.

On the final edge, $\bmy=(3t-2,0,3(1-t))$. Adding $2y_2/3$
to each term of the minimum yields
\[
S_\delta(\bmy)
=\min\biggl\{y_0+\frac{2y_2}{3},
             \frac{1-\delta}{3\delta}y_2\biggr\}
=\min\biggl\{t,\frac{(1-\delta)(1-t)}{\delta}\biggr\}.
\]
\blue{The first term gives the offset up to $t=1-\delta$, where
the two terms agree, and the second gives its linear decrease
to zero at $t=1$.}
Thus $S_\delta(\bm z_\ell(x))=(x-\psi_\ell(x))/s$ for every
$x\in[0,1]$, which proves \eqref{eq:phase-rounding-identity}.
\end{proof}

\FloatBarrier
We now use Lemmas~\ref{lem:phase} and~\ref{lem:rounding} to construct
and decode the cell labels at successive scales within a single
sequence of hidden layers.

\begin{proof}[Proof of Proposition~\ref{lem:pipeline}]
Fix the parameter sequences $\delta_\ell,u_\ell,v_\ell,w_\ell$
in the proposition. We construct the label $k_\ell$ as an affine
readout of layer $\ell+2$. The next two layers decode this label,
giving both the approximation $\Phi_\ell$ and the correction
$\Gamma_\ell$ as affine readouts of layer $\ell+4$.

Reserve the first two neurons of each hidden layer for the inner
and outer sines, respectively. For each input coordinate, allocate
$3N$ \ReLU{} neurons to phase refinement, one to rounding, and one
to copying the input coordinate. 
Two further \ReLU{} neurons store the positive and negative parts of
the preceding approximation $\Phi_{\ell-1}$, allowing the correction
head to compute $\Gamma_\ell=\Phi_\ell-\Phi_{\ell-1}$ for $\ell\ge1$.
The total width is
\[
\underbrace{2}_{\text{sines}}
+\underbrace{3dN}_{\text{phases}}
+\underbrace{d}_{\text{rounding}}
+\underbrace{d}_{\text{input copies}}
+\underbrace{2}_{\text{storage}}
=3dN+2d+4.
\]
In each layer, we set the incoming weights and bias of every unused
neuron to zero.

\mystep{1}{Initialize the phase and retain the input.}
Let $\bmx=(x_1,\ldots,x_d)\in[0,1]^d$ and fix $j\in\{1,\ldots,d\}$.
In the first hidden layer, three phase neurons for coordinate $x_j$
compute $\ReLU(x_j)$, $\ReLU(x_j-1/3)$, and $\ReLU(x_j-2/3)$.
Since $\ReLU(x_j)=x_j$, their outputs give the affine readout
\begin{equation}\label{eq:phase-initialization}
\bm z_0(x_j)=
\begin{bmatrix}
1-3x_j+3\ReLU(x_j-1/3)+3\ReLU(x_j-2/3)\\[4pt]
3x_j-6\ReLU(x_j-1/3)+3\ReLU(x_j-2/3)\\[4pt]
3\ReLU(x_j-1/3)-6\ReLU(x_j-2/3)
\end{bmatrix}.
\end{equation}
This is the coordinate formula for $\bm z(x_j)=\bm z_0(x_j)$ in
\eqref{eq:triangle-phase}, expressed in terms of the slope changes
at $1/3$ and $2/3$. \blue{For example, the first coordinate of
$\bm z$ has slopes $-3,0,3$ on the three successive intervals.
Both slope changes are $3$, giving the coefficients of the two
shifted \ReLU{} terms in the first row.}
The remaining phase neurons for $x_j$ have output zero in this layer.

The input-copy neuron for $x_j$ also computes $\ReLU(x_j)=x_j$.
Each subsequent layer copies this nonnegative value. Applying the
same construction to all $d$ coordinates keeps $\bmx$ available
throughout the network.

\mystep{2}{Refine the phase and read the current label.}
For the fixed coordinate $x_j$, the first hidden layer provides
$\bm z_0(x_j)$ as an affine readout by \eqref{eq:phase-initialization}.
For all subsequent layers, Lemma~\ref{lem:phase} gives a fixed affine
map $\bm{\calA}^{\mathrm{phase}}_N\in\aff(3,3N)$ and a fixed matrix
$\bmW_N^{\mathrm{phase}}\in\R^{3\times 3N}$ for the phase update.
Given the readout $\bmy=(y_0,y_1,y_2)=\bm z_\ell(x_j)$ from layer
$\ell+1$, the $3N$ phase neurons in layer $\ell+2$ compute
\[
\bm F_N(\bmy)
=\ReLU\bigl(\bm{\calA}^{\mathrm{phase}}_N(\bmy)\bigr),
\]
while the rounding neuron uses the same $\bmy$ in parallel.
The first three entries of $\bm F_N(\bmy)$ reproduce the current
phase $\bmy$, and \eqref{eq:phase-readouts} gives the next phase
as the linear readout
\[
\bmW_N^{\mathrm{phase}}\bm F_N(\bmy)
=\bm z(N^{\ell+1}x_j)=\bm z_{\ell+1}(x_j).
\]

The rounding neuron in layer $\ell+2$ computes
\[
H=\ReLU\biggl(y_0-\frac{1-3\delta_\ell}{3\delta_\ell}y_2\biggr).
\]
The same layer also contains $x_j,y_0,y_1,y_2$. With $s=N^{-\ell}$,
Lemma~\ref{lem:rounding} therefore gives the rounded coordinate as
the affine readout
\[
\psi_\ell(x_j)
=x_j-s\biggl(y_0+\frac{y_1+2y_2}{3}-H\biggr).
\]
Combining these readouts over all $d$ coordinates according to
\eqref{eq:scalar-label} gives the label at layer $\ell+2$,
\[
k_\ell(\bmx)
=\calA_\ell^{\mathrm{lab}}
  \bigl(\psi_\ell(x_1),\ldots,\psi_\ell(x_d)\bigr)
=1+\sum_{j=1}^d N^{j\ell}\psi_\ell(x_j).
\]

\blue{For example, the inner sine in layer $\ell+3$ requires the
preactivation $w_\ell k_\ell(\bmx)$. Set its bias to $w_\ell$ and,
for each $j$, assign the five outputs $(x_j,y_0,y_1,y_2,H)$
associated with $x_j$ the weights
\[
w_\ell N^{j\ell}(1,-s,-s/3,-2s/3,s).
\]
These five outputs contribute $w_\ell N^{j\ell}\psi_\ell(x_j)$.
Summing over $j$ and adding the bias gives the required preactivation
in a single row of the affine map. The strip width $\delta_\ell$
enters only the rounding neuron for each coordinate. Phase refinement
is exact throughout the cell, including the transition strip.}

\mystep{3}{Decode the labels at successive scales.}
Write $\bm{\mathcal V}_r$ for the hidden-state map after activation
in layer $r$, and $(\bm{\mathcal V}_r)_i$ for its $i$th component.
The first sine neuron in layer $\ell+3$ takes $w_\ell k_\ell$ as
its input. The second sine neuron in layer $\ell+4$ takes the output
of that neuron scaled by $v_\ell$. Thus
\[
\bigl(\bm{\mathcal V}_{\ell+3}\bigr)_1=\sin(w_\ell k_\ell),
\qquad
\bigl(\bm{\mathcal V}_{\ell+4}\bigr)_2
=\sin\Bigl(v_\ell\bigl(\bm{\mathcal V}_{\ell+3}\bigr)_1\Bigr).
\]
By \eqref{eq:decoded-value}, scaling the outer sine output by $u_\ell$
gives the affine readout
\begin{equation}\label{eq:current-affine-readout}
\Phi_\ell=u_\ell\bigl(\bm{\mathcal V}_{\ell+4}\bigr)_2.
\end{equation}

Layer $\ell+3$ provides both the inner sine $\sin(w_\ell k_\ell)$
and the next label $k_{\ell+1}$. Layer $\ell+4$ uses the former
to complete the decoding at level $\ell$ and the latter to start
decoding at level $\ell+1$. Thus the first four layers give $\Phi_0$,
and each additional layer gives the next approximation.
Table~\ref{tab:startup-readouts} shows how these operations overlap
in the first five layers.

\begin{table}[!htbp]
\centering
\caption{Phases, labels, and approximations available after activation
in the first five hidden layers. The stored approximation is recovered
from its positive and negative parts.}
\label{tab:startup-readouts}
\begingroup
\setlength{\tabcolsep}{16pt}
\renewcommand{\arraystretch}{1.152}
\resizebox{.82\linewidth}{!}{%
\begin{tabular}{lccccc}
\toprule
Hidden layer & 1 & 2 & 3 & 4 & 5 \\
\midrule
Phase (per input coordinate)
  & $\bm z_0$ & $\bm z_1$ & $\bm z_2$ & $\bm z_3$ & $\bm z_4$ \\
Cell label
  & --- & $k_0$ & $k_1$ & $k_2$ & $k_3$ \\
Level of inner sine
  & --- & --- & $0$ & $1$ & $2$ \\
Current approximation
  & --- & --- & --- & $\Phi_0$ & $\Phi_1$ \\
Stored preceding approximation
  & --- & --- & --- & --- & $\Phi_0$ \\
\bottomrule
\end{tabular}}
\endgroup
\end{table}

\mystep{4}{Read the corrections and extend the network.}
At level zero, $\Gamma_0=\Phi_0$, so
\eqref{eq:current-affine-readout} gives the first correction head.
The storage neurons have output zero through layer four. For
$\ell\ge1$, layer $\ell+3$ provides the preceding approximation
$\Phi_{\ell-1}=u_{\ell-1}(\bm{\mathcal V}_{\ell+3})_2$.
Use $\Phi_{\ell-1}$ and $-\Phi_{\ell-1}$ as the preactivations of
the two storage neurons in layer $\ell+4$. Their outputs have
difference $\Phi_{\ell-1}$, so the correction head is
\begin{equation}\label{eq:affine-correction-head}
\begin{aligned}
\calA_\ell^{\mathrm{out}}\circ\bm{\mathcal V}_{\ell+4}
&=u_\ell\bigl(\bm{\mathcal V}_{\ell+4}\bigr)_2
  -\ReLU(\Phi_{\ell-1})+\ReLU(-\Phi_{\ell-1})\\
&=\Phi_\ell-\Phi_{\ell-1}=\Gamma_\ell.
\end{aligned}
\end{equation}
\blue{Thus layer $\ell+4$ provides both $\Phi_\ell$ and $\Gamma_\ell$
as affine readouts.} Summing the corrections gives
\[
\sum_{r=0}^{\ell}
\calA_r^{\mathrm{out}}\circ\bm{\mathcal V}_{r+4}
=\Phi_0+\sum_{r=1}^{\ell}(\Phi_r-\Phi_{r-1})
=\Phi_\ell.
\]

Each affine readout used by the next layer can be substituted
directly into that layer's preactivation. Thus the preactivations
specify the rows and biases of a single affine map $\bm{\calA}_r$
at each layer. The first map acts on $\R^d$, and each subsequent
map acts on $\R^{3dN+2d+4}$. The activation $\bmvarrho$ applies
sine to the first two coordinates and \ReLU{} to the remaining
coordinates. Table~\ref{tab:shared-layer} summarizes the operations
in layer $\ell+4$ for $\ell\ge1$.

\begin{table}[!htbp]
\centering
\caption{Operations in layer $\ell+4$ for $\ell\ge1$.
Phase refinement and rounding are performed for each input coordinate.}
\label{tab:shared-layer}
\begingroup
\setlength{\tabcolsep}{16pt}
\renewcommand{\arraystretch}{1.152}
\resizebox{.825\linewidth}{!}{%
\begin{tabular}{llc}
\toprule
Channels & Operation & Count \\
\midrule
Inner sine
  & Compute $\sin(w_{\ell+1}k_{\ell+1})$ & $1$ \\
Outer sine
  & Give $\Phi_\ell$ after scaling by $u_\ell$ & $1$ \\
Phase
  & Refine $\bm z_{\ell+2}(x_j)$ to $\bm z_{\ell+3}(x_j)$
    for each $j$ & $3dN$ \\
Rounding
  & Supply one hinge per coordinate for $k_{\ell+2}$ & $d$ \\
Input copies
  & Retain $\bmx$ & $d$ \\
Previous output
  & Store $\ReLU(\Phi_{\ell-1})$ and $\ReLU(-\Phi_{\ell-1})$ & $2$ \\
\midrule
\multicolumn{2}{l}{Total hidden width} & $3dN+2d+4$ \\
\bottomrule
\end{tabular}}
\endgroup
\end{table}


Finally, fix $m\ge0$. Layer $m+4$ already provides $\Phi_m$ and
the inner sine $\sin(w_{m+1}k_{m+1})$. Thus $w_{m+1}$ is already
fixed at this layer, as is $\delta_{m+2}$ used to form $k_{m+2}$.
\blue{Applying the same construction in layer $m+5$ completes the
decoding of $\Phi_{m+1}$ and retains the positive and negative parts
of $\Phi_m$, so the new correction head gives
$\Gamma_{m+1}=\Phi_{m+1}-\Phi_m$.}
Repeating this extension to any level $n>m$ preserves the first
$m+4$ hidden maps and all correction heads through level $m$.
This proves the proposition.
\end{proof}

\FloatBarrier

\section{Conclusion}\label{sec:conclusion}


For every $f\in L^p([0,1]^d)$ with $1\le p<\infty$ and every integer
$N\ge2$, we have constructed a single network of width $3dN+2d+4$
without prescribing its terminal depth. The first four hidden layers
give an approximation at scale one. 
Each additional layer
refines the scale by a factor of $N$ and gives both the new
approximation and its correction to the preceding one. The error
at level $\ell$ is bounded by $(2d+1)\omega_{f,p}(N^{-\ell})$.
This gives convergence for every such $f$ and a geometric bound
when the translation modulus satisfies a H\"older estimate.

There are two reasons why the error bound depends only on the
current scale. First, the current readout fits the averages of $f$
on the current grid. Its error can be estimated directly against
$\bar f_\ell$, even though the same readout can also be written
as the sum of all preceding corrections and the current one.
Second, the phase is updated exactly. A point that falls in a
transition strip at one level can therefore receive the correct
cell label at the next. These facts let us choose the decoding
accuracy and the strip width separately at every level. Under the
H\"older bound $\,\omega_{f,p}(t)\le\lambda_p t^\alpha$, for example, the
resulting upper bound is $(2d+1)\lambda_p N^{-\alpha\ell}$ for the same
sequence of hidden maps.

Two questions arise from this result. The first concerns the cost of decoding the cell averages. At a given scale the interpolation argument fits a finite set of values with a small number of real coefficients, without quantifying the frequency or precision required to do so. How much additional width or depth would permit useful bounds on both quantities while preserving the earlier readouts? The second question concerns learning. 
We choose the coefficients using the cell averages of the target, whereas multigrade training fits each new correction to the residual.
It remains to determine conditions under which that fitting procedure achieves comparable bounds at successive scales with the earlier grades held fixed.

\section*{Acknowledgments}
Shijun Zhang was partially supported by the start-up fund P0053092 from The Hong Kong Polytechnic University.
Zuowei Shen was partially supported under the Distinguished Professorship of the National University of Singapore.
Yuesheng Xu was supported in part by the U.S. National Science Foundation under Grant DMS-2208386.

\setlength{\bibsep}{2pt plus 0.3pt}
\bibliographystyle{plainnat}
\bibliography{references}

\begin{thebibliography}{65}
\providecommand{\natexlab}[1]{#1}
\providecommand{\url}[1]{\texttt{#1}}
\expandafter\ifx\csname urlstyle\endcsname\relax
  \providecommand{\doi}[1]{doi: #1}\else
  \providecommand{\doi}{doi: \begingroup \urlstyle{rm}\Url}\fi

\bibitem[Arora et~al.(2022)Arora, Li, and Panigrahi]{arora2022understanding}
Sanjeev Arora, Zhiyuan Li, and Abhishek Panigrahi.
\newblock Understanding gradient descent on the edge of stability in deep
  learning.
\newblock In Kamalika Chaudhuri, Stefanie Jegelka, Le~Song, Csaba Szepesvari,
  Gang Niu, and Sivan Sabato, editors, \emph{Proceedings of the 39th
  International Conference on Machine Learning}, volume 162 of
  \emph{Proceedings of Machine Learning Research}, pages 948--1024. PMLR,
  17--23 Jul 2022.
\newblock URL: \url{https://proceedings.mlr.press/v162/arora22a.html}.

\bibitem[Bengio et~al.(2006)Bengio, Lamblin, Popovici, and
  Larochelle]{Bengio2007}
Yoshua Bengio, Pascal Lamblin, Dan Popovici, and Hugo Larochelle.
\newblock Greedy layer-wise training of deep networks.
\newblock In B.~Sch\"{o}lkopf, J.~Platt, and T.~Hoffman, editors,
  \emph{Advances in Neural Information Processing Systems}, volume~19. MIT
  Press, 2006.
\newblock URL:
  \url{https://proceedings.neurips.cc/paper_files/paper/2006/file/5da713a690c067105aeb2fae32403405-Paper.pdf}.

\bibitem[B{\"o}lcskei et~al.(2019)B{\"o}lcskei, Grohs, Kutyniok, and
  Petersen]{B_lcskei_2019}
Helmut B{\"o}lcskei, Philipp Grohs, Gitta Kutyniok, and Philipp Petersen.
\newblock Optimal approximation with sparsely connected deep neural networks.
\newblock \emph{SIAM Journal on Mathematics of Data Science}, 1\penalty0
  (1):\penalty0 8--45, Jan 2019.
\newblock ISSN 2577-0187.
\newblock \doi{10.1137/18m118709x}.

\bibitem[Cai et~al.(2020)Cai, Li, and Liu]{doi:10.1137/19M1310050}
Wei Cai, Xiaoguang Li, and Lizuo Liu.
\newblock A phase shift deep neural network for high frequency approximation
  and wave problems.
\newblock \emph{SIAM Journal on Scientific Computing}, 42\penalty0
  (5):\penalty0 A3285--A3312, 2020.
\newblock \doi{10.1137/19M1310050}.

\bibitem[Cheng et~al.(2025)Cheng, Li, Lin, and Shen]{doi:10.1137/23M1599744}
Jingpu Cheng, Qianxiao Li, Ting Lin, and Zuowei Shen.
\newblock Interpolation, approximation, and controllability of deep neural
  networks.
\newblock \emph{SIAM Journal on Control and Optimization}, 63\penalty0
  (1):\penalty0 625--649, 2025.
\newblock \doi{10.1137/23M1599744}.

\bibitem[{Cheng} et~al.(2026){Cheng}, {Li}, {Lin}, and
  {Shen}]{2026arXiv260315363C}
Jingpu {Cheng}, Qianxiao {Li}, Ting {Lin}, and Zuowei {Shen}.
\newblock Deep learning and the rate of approximation by flows.
\newblock \emph{arXiv e-prints}, art. arXiv:2603.15363, March 2026.
\newblock \doi{10.48550/arXiv.2603.15363}.

\bibitem[{Cohen} et~al.(2022){Cohen}, {DeVore}, {Petrova}, and
  {Wojtaszczyk}]{cohen2020optimal}
Albert {Cohen}, Ronald {DeVore}, Guergana {Petrova}, and Przemyslaw
  {Wojtaszczyk}.
\newblock Optimal stable nonlinear approximation.
\newblock \emph{Foundations of Computational Mathematics}, 22:\penalty0
  607--648, 2022.
\newblock \doi{10.1007/s10208-021-09494-z}.

\bibitem[Cohen et~al.(2021)Cohen, Kaur, Li, Kolter, and
  Talwalkar]{cohen2021gradient}
Jeremy~M. Cohen, Simran Kaur, Yuanzhi Li, J.~Zico Kolter, and Ameet Talwalkar.
\newblock Gradient descent on neural networks typically occurs at the edge of
  stability.
\newblock In \emph{International Conference on Learning Representations
  (ICLR)}, 2021.
\newblock URL: \url{https://openreview.net/forum?id=jh-rTtvkGeM}.

\bibitem[Cybenko(1989)]{Cybenko1989ApproximationBS}
George Cybenko.
\newblock Approximation by superpositions of a sigmoidal function.
\newblock \emph{Mathematics of Control, Signals, and Systems}, 2:\penalty0
  303--314, 1989.
\newblock \doi{10.1007/BF02551274}.

\bibitem[Daubechies et~al.(2022)Daubechies, DeVore, Foucart, Hanin, and
  Petrova]{Ingrid}
Ingrid Daubechies, Ronald DeVore, Simon Foucart, Boris Hanin, and Guergana
  Petrova.
\newblock Nonlinear approximation and (deep) $\mathrm{ReLU}$ networks.
\newblock \emph{Constructive Approximation}, 55:\penalty0 127--172, 2022.
\newblock \doi{10.1007/s00365-021-09548-z}.

\bibitem[DeVore(1998)]{devore_1998}
Ronald~A. DeVore.
\newblock Nonlinear approximation.
\newblock \emph{Acta Numerica}, 7:\penalty0 51–150, 1998.
\newblock \doi{10.1017/S0962492900002816}.

\bibitem[DeVore and Lorentz(1993)]{DeVoreLorentz1993}
Ronald~A. DeVore and George~G. Lorentz.
\newblock \emph{Constructive Approximation}, volume 303 of \emph{Grundlehren
  der mathematischen Wissenschaften}.
\newblock Springer, Berlin, 1993.

\bibitem[Ding et~al.(2023)Ding, Xia, and Bu]{NEURIPS2023_1d5a9286}
Shizhe Ding, Boyang Xia, and Dongbo Bu.
\newblock Accurate interpolation for scattered data through hierarchical
  residual refinement.
\newblock In A.~Oh, T.~Naumann, A.~Globerson, K.~Saenko, M.~Hardt, and
  S.~Levine, editors, \emph{Advances in Neural Information Processing Systems},
  volume~36, pages 9144--9155. Curran Associates, Inc., 2023.
\newblock URL:
  \url{https://proceedings.neurips.cc/paper_files/paper/2023/file/1d5a92867cf463fad136cfa23395840b-Paper-Conference.pdf}.

\bibitem[Fan et~al.(2022)Fan, Wang, Guo, Zhu, Yan, Wang, and Yu]{9614997}
Feng-Lei Fan, Dayang Wang, Hengtao Guo, Qikui Zhu, Pingkun Yan, Ge~Wang, and
  Hengyong Yu.
\newblock On a sparse shortcut topology of artificial neural networks.
\newblock \emph{IEEE Transactions on Artificial Intelligence}, 3\penalty0
  (4):\penalty0 595--608, 2022.
\newblock \doi{10.1109/TAI.2021.3128132}.

\bibitem[Fang and Xu(2024)]{FangXu2024}
Ronglong Fang and Yuesheng Xu.
\newblock Addressing spectral bias of deep neural networks by multi-grade deep
  learning.
\newblock In A.~Globerson, L.~Mackey, D.~Belgrave, A.~Fan, U.~Paquet,
  J.~Tomczak, and C.~Zhang, editors, \emph{Advances in Neural Information
  Processing Systems}, volume~37, pages 114122--114146. Curran Associates,
  Inc., 2024.
\newblock \doi{10.52202/079017-3625}.

\bibitem[{Fang} and {Xu}(2025)]{FangXu2025}
Ronglong {Fang} and Yuesheng {Xu}.
\newblock Computational advantages of multi-grade deep learning: Convergence
  analysis and performance insights.
\newblock \emph{arXiv e-prints}, art. arXiv:2507.20351, July 2025.
\newblock \doi{10.48550/arXiv.2507.20351}.

\bibitem[Gallant and White(1988)]{Gallant1988ThereEA}
A.~Ronald Gallant and Halbert White.
\newblock There exists a neural network that does not make avoidable mistakes.
\newblock In \emph{IEEE 1988 International Conference on Neural Networks},
  pages 657--664 vol.1, 1988.
\newblock \doi{10.1109/ICNN.1988.23903}.

\bibitem[He et~al.(2020)He, Li, Xu, and Zheng]{HeLiXuZheng2020}
Juncai He, Lin Li, Jinchao Xu, and Chunyue Zheng.
\newblock {ReLU} deep neural networks and linear finite elements.
\newblock \emph{Journal of Computational Mathematics}, 38\penalty0
  (3):\penalty0 502--527, 2020.
\newblock \doi{10.4208/jcm.1901-m2018-0160}.

\bibitem[He et~al.(2016)He, Zhang, Ren, and Sun]{7780459}
Kaiming He, Xiangyu Zhang, Shaoqing Ren, and Jian Sun.
\newblock Deep residual learning for image recognition.
\newblock In \emph{2016 IEEE Conference on Computer Vision and Pattern
  Recognition (CVPR)}, pages 770--778, Los Alamitos, CA, USA, June 2016. IEEE
  Computer Society.
\newblock \doi{10.1109/CVPR.2016.90}.

\bibitem[Hinton et~al.(2006)Hinton, Osindero, and Teh]{6796673}
Geoffrey~E. Hinton, Simon Osindero, and Yee-Whye Teh.
\newblock A fast learning algorithm for deep belief nets.
\newblock \emph{Neural Computation}, 18\penalty0 (7):\penalty0 1527--1554,
  2006.
\newblock \doi{10.1162/neco.2006.18.7.1527}.

\bibitem[Hornik(1991)]{HORNIK1991251}
Kurt Hornik.
\newblock Approximation capabilities of multilayer feedforward networks.
\newblock \emph{Neural Networks}, 4\penalty0 (2):\penalty0 251--257, 1991.
\newblock ISSN 0893-6080.
\newblock \doi{10.1016/0893-6080(91)90009-T}.

\bibitem[Hornik et~al.(1989)Hornik, Stinchcombe, and White]{HORNIK1989359}
Kurt Hornik, Maxwell Stinchcombe, and Halbert White.
\newblock Multilayer feedforward networks are universal approximators.
\newblock \emph{Neural Networks}, 2\penalty0 (5):\penalty0 359--366, 1989.
\newblock ISSN 0893-6080.
\newblock \doi{10.1016/0893-6080(89)90020-8}.

\bibitem[Jiang and Xu(2026)]{Jiang-Xu2025}
Jie Jiang and Yuesheng Xu.
\newblock Adaptive multi-grade deep learning for highly oscillatory {F}redholm
  integral equations of the second kind.
\newblock \emph{Journal of Scientific Computing}, 106\penalty0 (3):\penalty0
  64, 2026.
\newblock \doi{10.1007/s10915-026-03189-9}.

\bibitem[Jiao et~al.(2023)Jiao, Lai, Lu, Wang, Yang, and
  Yang]{doi:10.1137/21M144431X}
Yuling Jiao, Yanming Lai, Xiliang Lu, Fengru Wang, Jerry~Zhijian Yang, and
  Yuanyuan Yang.
\newblock Deep neural networks with {ReLU-Sine-Exponential} activations break
  curse of dimensionality in approximation on {H\"older} class.
\newblock \emph{SIAM Journal on Mathematical Analysis}, 55\penalty0
  (4):\penalty0 3635--3649, 2023.
\newblock \doi{10.1137/21M144431X}.

\bibitem[Kolmogorov and Tikhomirov(1959)]{KolmogorovTikhomirov1959}
Andrey~N. Kolmogorov and Vladimir~M. Tikhomirov.
\newblock {$\varepsilon$}-entropy and {$\varepsilon$}-capacity of sets in
  function spaces.
\newblock \emph{Uspekhi Matematicheskikh Nauk}, 14\penalty0 (2(86)):\penalty0
  3--86, 1959.
\newblock URL: \url{https://www.mathnet.ru/eng/rm7289}.
\newblock English translation: American Mathematical Society Translations,
  Series 2, 17 (1961), 277--364.

\bibitem[{Li} et~al.(2026){Li}, {Shen}, {Yang}, and
  {Zhang}]{LiShenYangZhang2026Tradeoffs}
Baicheng {Li}, Zuowei {Shen}, Haizhao {Yang}, and Shijun {Zhang}.
\newblock Optimal tradeoffs between network size and parameter magnitude in
  neural approximation and minimax regression.
\newblock \emph{arXiv e-prints}, art. arXiv:2609.25710, September 2026.
\newblock \doi{10.48550/arXiv.2609.25710}.

\bibitem[Li et~al.(2023)Li, Lin, and Shen]{LiLinShen2023DynamicalSystems}
Qianxiao Li, Ting Lin, and Zuowei Shen.
\newblock Deep learning via dynamical systems: An approximation perspective.
\newblock \emph{Journal of the European Mathematical Society}, 25\penalty0
  (5):\penalty0 1671--1709, 2023.
\newblock \doi{10.4171/JEMS/1221}.

\bibitem[Lin and Jegelka(2018)]{NEURIPS2018_03bfc1d4}
Hongzhou Lin and Stefanie Jegelka.
\newblock {ResNet} with one-neuron hidden layers is a universal approximator.
\newblock In S.~Bengio, H.~Wallach, H.~Larochelle, K.~Grauman, N.~Cesa-Bianchi,
  and R.~Garnett, editors, \emph{Advances in Neural Information Processing
  Systems}, volume~31. Curran Associates, Inc., 2018.
\newblock URL:
  \url{https://proceedings.neurips.cc/paper/2018/file/03bfc1d4783966c69cc6aef8247e0103-Paper.pdf}.

\bibitem[Lu et~al.(2021)Lu, Shen, Yang, and Zhang]{shijun:smooth:functions}
Jianfeng Lu, Zuowei Shen, Haizhao Yang, and Shijun Zhang.
\newblock Deep network approximation for smooth functions.
\newblock \emph{SIAM Journal on Mathematical Analysis}, 53\penalty0
  (5):\penalty0 5465--5506, 2021.
\newblock \doi{10.1137/20M134695X}.

\bibitem[Luo et~al.(2021)Luo, Ma, Xu, and Zhang]{luo2019theory}
Tao Luo, Zheng Ma, Zhi-Qin~John Xu, and Yaoyu Zhang.
\newblock Theory of the frequency principle for general deep neural networks.
\newblock \emph{CSIAM Transactions on Applied Mathematics}, 2\penalty0
  (3):\penalty0 484--507, 2021.
\newblock ISSN 2708-0579.
\newblock \doi{10.4208/csiam-am.SO-2020-0005}.

\bibitem[Maiorov and Pinkus(1999)]{MAIOROV199981}
Vitaly Maiorov and Allan Pinkus.
\newblock Lower bounds for approximation by {MLP} neural networks.
\newblock \emph{Neurocomputing}, 25\penalty0 (1):\penalty0 81--91, 1999.
\newblock ISSN 0925-2312.
\newblock \doi{10.1016/S0925-2312(98)00111-8}.

\bibitem[Maiorov and Ratsaby(1999)]{MaiorovRatsaby1999}
Vitaly Maiorov and Joel Ratsaby.
\newblock On the degree of approximation by manifolds of finite
  pseudo-dimension.
\newblock \emph{Constructive Approximation}, 15\penalty0 (2):\penalty0
  291--300, 1999.
\newblock \doi{10.1007/s003659900108}.

\bibitem[Mallat(2009)]{WaveletTour2009}
St{\'e}phane Mallat.
\newblock \emph{A Wavelet Tour of Signal Processing: The Sparse Way}.
\newblock Academic Press, Orlando, FL, USA, 3rd edition, 2009.
\newblock ISBN 0123743702, 9780123743701.
\newblock \doi{10.1016/B978-0-12-374370-1.X0001-8}.

\bibitem[Oreshkin et~al.(2020)Oreshkin, Carpov, Chapados, and
  Bengio]{oreshkin2020nbeats}
Boris Oreshkin, Dmitri Carpov, Nicolas Chapados, and Yoshua Bengio.
\newblock {N-BEATS}: Neural basis expansion analysis for interpretable time
  series forecasting.
\newblock In \emph{International Conference on Learning Representations
  (ICLR)}, 2020.
\newblock URL: \url{https://openreview.net/forum?id=r1ecqn4YwB}.

\bibitem[Ou and B{\"o}lcskei(2026)]{OuBolcskei2026Covering}
Weigutian Ou and Helmut B{\"o}lcskei.
\newblock Covering numbers for deep {ReLU} networks with applications to
  function approximation and nonparametric regression.
\newblock \emph{Foundations of Computational Mathematics}, 2026.
\newblock \doi{10.1007/s10208-026-09762-w}.

\bibitem[{Ou} et~al.(2024){Ou}, {Schenkel}, and
  {B{\"o}lcskei}]{OuSchenkelBolcskei2024Quantization}
Weigutian {Ou}, Philipp {Schenkel}, and Helmut {B{\"o}lcskei}.
\newblock Three quantization regimes for {ReLU} networks.
\newblock \emph{arXiv e-prints}, art. arXiv:2405.01952, May 2024.
\newblock \doi{10.48550/arXiv.2405.01952}.

\bibitem[Pinkus(1999)]{Pinkus1999MLP}
Allan Pinkus.
\newblock Approximation theory of the {MLP} model in neural networks.
\newblock \emph{Acta Numerica}, 8:\penalty0 143--195, 1999.
\newblock \doi{10.1017/S0962492900002919}.

\bibitem[Rahaman et~al.(2019)Rahaman, Baratin, Arpit, Draxler, Lin, Hamprecht,
  Bengio, and Courville]{rahaman2019spectral}
Nasim Rahaman, Aristide Baratin, Devansh Arpit, Felix Draxler, Min Lin, Fred
  Hamprecht, Yoshua Bengio, and Aaron Courville.
\newblock On the spectral bias of neural networks.
\newblock In Kamalika Chaudhuri and Ruslan Salakhutdinov, editors,
  \emph{Proceedings of the 36th International Conference on Machine Learning},
  volume~97 of \emph{Proceedings of Machine Learning Research}, pages
  5301--5310. PMLR, 09--15 Jun 2019.
\newblock URL: \url{https://proceedings.mlr.press/v97/rahaman19a.html}.

\bibitem[Ron and Shen(1997)]{RON1997408}
Amos Ron and Zuowei Shen.
\newblock Affine systems in ${L}^2(\mathbb{R}^d)$: The analysis of the analysis
  operator.
\newblock \emph{Journal of Functional Analysis}, 148\penalty0 (2):\penalty0
  408--447, 1997.
\newblock ISSN 0022-1236.
\newblock \doi{10.1006/jfan.1996.3079}.

\bibitem[Shen et~al.(2019)Shen, Yang, and Zhang]{shijun:NonlineArpprox}
Zuowei Shen, Haizhao Yang, and Shijun Zhang.
\newblock Nonlinear approximation via compositions.
\newblock \emph{Neural Networks}, 119:\penalty0 74--84, 2019.
\newblock ISSN 0893-6080.
\newblock \doi{10.1016/j.neunet.2019.07.011}.

\bibitem[Shen et~al.(2020)Shen, Yang, and
  Zhang]{shijun:Characterized:by:Numer:Neurons}
Zuowei Shen, Haizhao Yang, and Shijun Zhang.
\newblock Deep network approximation characterized by number of neurons.
\newblock \emph{Communications in Computational Physics}, 28\penalty0
  (5):\penalty0 1768--1811, 2020.
\newblock ISSN 1991-7120.
\newblock \doi{10.4208/cicp.OA-2020-0149}.

\bibitem[Shen et~al.(2021{\natexlab{a}})Shen, Yang, and
  Zhang]{shijun:floor:relu}
Zuowei Shen, Haizhao Yang, and Shijun Zhang.
\newblock Deep network with approximation error being reciprocal of width to
  power of square root of depth.
\newblock \emph{Neural Computation}, 33\penalty0 (4):\penalty0 1005--1036, 03
  2021{\natexlab{a}}.
\newblock ISSN 0899-7667.
\newblock \doi{10.1162/neco_a_01364}.

\bibitem[Shen et~al.(2021{\natexlab{b}})Shen, Yang, and
  Zhang]{shijun:three:layers}
Zuowei Shen, Haizhao Yang, and Shijun Zhang.
\newblock Neural network approximation: {T}hree hidden layers are enough.
\newblock \emph{Neural Networks}, 141:\penalty0 160--173, 2021{\natexlab{b}}.
\newblock ISSN 0893-6080.
\newblock \doi{10.1016/j.neunet.2021.04.011}.

\bibitem[Shen et~al.(2022{\natexlab{a}})Shen, Yang, and
  Zhang]{shijun:arbitrary:error:with:fixed:size}
Zuowei Shen, Haizhao Yang, and Shijun Zhang.
\newblock Deep network approximation: Achieving arbitrary accuracy with fixed
  number of neurons.
\newblock \emph{Journal of Machine Learning Research}, 23\penalty0
  (276):\penalty0 1--60, 2022{\natexlab{a}}.
\newblock URL: \url{http://jmlr.org/papers/v23/21-1404.html}.

\bibitem[Shen et~al.(2022{\natexlab{b}})Shen, Yang, and
  Zhang]{shijun:intrinsic:parameters}
Zuowei Shen, Haizhao Yang, and Shijun Zhang.
\newblock Deep network approximation in terms of intrinsic parameters.
\newblock In Kamalika Chaudhuri, Stefanie Jegelka, Le~Song, Csaba Szepesvari,
  Gang Niu, and Sivan Sabato, editors, \emph{Proceedings of the 39th
  International Conference on Machine Learning}, volume 162 of
  \emph{Proceedings of Machine Learning Research}, pages 19909--19934. PMLR,
  17--23 Jul 2022{\natexlab{b}}.
\newblock URL: \url{https://proceedings.mlr.press/v162/shen22g.html}.

\bibitem[Shen et~al.(2022{\natexlab{c}})Shen, Yang, and
  Zhang]{shijun:net:arc:beyond:width:depth}
Zuowei Shen, Haizhao Yang, and Shijun Zhang.
\newblock Neural network architecture beyond width and depth.
\newblock In S.~Koyejo, S.~Mohamed, A.~Agarwal, D.~Belgrave, K.~Cho, and A.~Oh,
  editors, \emph{Advances in Neural Information Processing Systems}, volume~35,
  pages 5669--5681. Curran Associates, Inc., 2022{\natexlab{c}}.
\newblock URL:
  \url{https://proceedings.neurips.cc/paper_files/paper/2022/hash/257be12f31dfa7cc158dda99822c6fd1-Abstract-Conference.html}.

\bibitem[Shen et~al.(2022{\natexlab{d}})Shen, Yang, and
  Zhang]{shijun:optimal:rate:in:width:and:depth}
Zuowei Shen, Haizhao Yang, and Shijun Zhang.
\newblock Optimal approximation rate of {ReLU} networks in terms of width and
  depth.
\newblock \emph{Journal de Mathématiques Pures et Appliquées}, 157:\penalty0
  101--135, 2022{\natexlab{d}}.
\newblock ISSN 0021-7824.
\newblock \doi{10.1016/j.matpur.2021.07.009}.

\bibitem[Siegel and Xu(2022)]{SIEGEL20221}
Jonathan~W. Siegel and Jinchao Xu.
\newblock High-order approximation rates for shallow neural networks with
  cosine and {ReLU}$^k$ activation functions.
\newblock \emph{Applied and Computational Harmonic Analysis}, 58:\penalty0
  1--26, 2022.
\newblock ISSN 1063-5203.
\newblock \doi{10.1016/j.acha.2021.12.005}.

\bibitem[Sitzmann et~al.(2020)Sitzmann, Martel, Bergman, Lindell, and
  Wetzstein]{NEURIPS2020_53c04118}
Vincent Sitzmann, Julien Martel, Alexander Bergman, David Lindell, and Gordon
  Wetzstein.
\newblock Implicit neural representations with periodic activation functions.
\newblock In H.~Larochelle, M.~Ranzato, R.~Hadsell, M.F. Balcan, and H.~Lin,
  editors, \emph{Advances in Neural Information Processing Systems}, volume~33,
  pages 7462--7473. Curran Associates, Inc., 2020.
\newblock URL:
  \url{https://proceedings.neurips.cc/paper_files/paper/2020/file/53c04118df112c13a8c34b38343b9c10-Paper.pdf}.

\bibitem[Tancik et~al.(2020)Tancik, Srinivasan, Mildenhall, Fridovich-Keil,
  Raghavan, Singhal, Ramamoorthi, Barron, and Ng]{NEURIPS2020_55053683}
Matthew Tancik, Pratul Srinivasan, Ben Mildenhall, Sara Fridovich-Keil, Nithin
  Raghavan, Utkarsh Singhal, Ravi Ramamoorthi, Jonathan Barron, and Ren Ng.
\newblock Fourier features let networks learn high frequency functions in low
  dimensional domains.
\newblock In H.~Larochelle, M.~Ranzato, R.~Hadsell, M.F. Balcan, and H.~Lin,
  editors, \emph{Advances in Neural Information Processing Systems}, volume~33,
  pages 7537--7547. Curran Associates, Inc., 2020.
\newblock URL:
  \url{https://proceedings.neurips.cc/paper_files/paper/2020/file/55053683268957697aa39fba6f231c68-Paper.pdf}.

\bibitem[Wang et~al.(2025)Wang, Zhang, Zeng, Xie, Guo, Zeng, and
  Fan]{WANG2025107258}
Qianchao Wang, Shijun Zhang, Dong Zeng, Zhaoheng Xie, Hengtao Guo, Tieyong
  Zeng, and Feng-Lei Fan.
\newblock Don’t fear peculiar activation functions: {EUAF} and beyond.
\newblock \emph{Neural Networks}, 186:\penalty0 107258, 2025.
\newblock ISSN 0893-6080.
\newblock \doi{10.1016/j.neunet.2025.107258}.

\bibitem[Xu(2026)]{Xu2023}
Yuesheng Xu.
\newblock Multi-grade deep learning.
\newblock \emph{Communications on Applied Mathematics and Computation},
  8:\penalty0 778--829, 2026.
\newblock ISSN 2661-8893.
\newblock \doi{10.1007/s42967-024-00474-y}.

\bibitem[Xu and Zeng(2023)]{XuZeng2023}
Yuesheng Xu and Taishan Zeng.
\newblock Multi-grade deep learning for partial differential equations with
  applications to the {B}urgers equation.
\newblock \emph{arXiv e-prints}, art. arXiv:2309.07401, September 2023.
\newblock \doi{10.48550/arXiv.2309.07401}.

\bibitem[Xu et~al.(2019)Xu, Zhang, and Xiao]{xu2019training}
Zhi-Qin~John Xu, Yaoyu Zhang, and Yanyang Xiao.
\newblock Training behavior of deep neural network in frequency domain.
\newblock In \emph{Neural Information Processing: 26th International
  Conference, ICONIP 2019, Sydney, NSW, Australia, December 12--15, 2019,
  Proceedings, Part I}, pages 264--274. Springer, 2019.
\newblock \doi{10.1007/978-3-030-36708-4_22}.

\bibitem[Xu et~al.(2020)Xu, Zhang, Luo, Xiao, and Ma]{xu2019frequency}
Zhi-Qin~John Xu, Yaoyu Zhang, Tao Luo, Yanyang Xiao, and Zheng Ma.
\newblock Frequency principle: Fourier analysis sheds light on deep neural
  networks.
\newblock \emph{Communications in Computational Physics}, 28\penalty0
  (5):\penalty0 1746--1767, 2020.
\newblock ISSN 1991-7120.
\newblock \doi{10.4208/cicp.OA-2020-0085}.

\bibitem[Yang et~al.(2022)Yang, Li, and Wang]{yang2020approximation}
Yunfei Yang, Zhen Li, and Yang Wang.
\newblock Approximation in shift-invariant spaces with deep {ReLU} neural
  networks.
\newblock \emph{Neural Networks}, 153:\penalty0 269--281, 2022.
\newblock ISSN 0893-6080.
\newblock \doi{10.1016/j.neunet.2022.06.013}.

\bibitem[Yarotsky(2017)]{yarotsky2017}
Dmitry Yarotsky.
\newblock Error bounds for approximations with deep {ReLU} networks.
\newblock \emph{Neural Networks}, 94:\penalty0 103--114, 2017.
\newblock ISSN 0893-6080.
\newblock \doi{10.1016/j.neunet.2017.07.002}.

\bibitem[Yarotsky(2018)]{yarotsky18a}
Dmitry Yarotsky.
\newblock Optimal approximation of continuous functions by very deep {ReLU}
  networks.
\newblock In S\'ebastien Bubeck, Vianney Perchet, and Philippe Rigollet,
  editors, \emph{Proceedings of the 31st Conference On Learning Theory},
  volume~75 of \emph{Proceedings of Machine Learning Research}, pages 639--649.
  PMLR, 06--09 Jul 2018.
\newblock URL: \url{http://proceedings.mlr.press/v75/yarotsky18a.html}.

\bibitem[Yarotsky(2021)]{pmlr-v139-yarotsky21a}
Dmitry Yarotsky.
\newblock Elementary superexpressive activations.
\newblock In Marina Meila and Tong Zhang, editors, \emph{Proceedings of the
  38th International Conference on Machine Learning}, volume 139 of
  \emph{Proceedings of Machine Learning Research}, pages 11932--11940. PMLR,
  18--24 Jul 2021.
\newblock URL: \url{https://proceedings.mlr.press/v139/yarotsky21a.html}.

\bibitem[Yarotsky and Zhevnerchuk(2020)]{yarotsky:2019:06}
Dmitry Yarotsky and Anton Zhevnerchuk.
\newblock The phase diagram of approximation rates for deep neural networks.
\newblock In H.~Larochelle, M.~Ranzato, R.~Hadsell, M.~F. Balcan, and H.~Lin,
  editors, \emph{Advances in Neural Information Processing Systems}, volume~33,
  pages 13005--13015. Curran Associates, Inc., 2020.
\newblock URL:
  \url{https://proceedings.neurips.cc/paper/2020/file/979a3f14bae523dc5101c52120c535e9-Paper.pdf}.

\bibitem[Zhang et~al.(2024)Zhang, Lu, and
  Zhao]{shijun:2023:beyond:ReLU:to:diverse:actfun}
Shijun Zhang, Jianfeng Lu, and Hongkai Zhao.
\newblock Deep network approximation: Beyond {ReLU} to diverse activation
  functions.
\newblock \emph{Journal of Machine Learning Research}, 25\penalty0
  (35):\penalty0 1--39, 2024.
\newblock URL: \url{http://jmlr.org/papers/v25/23-0912.html}.

\bibitem[{Zhang} et~al.(2026){Zhang}, {Zhao}, {Zhong}, and
  {Zhou}]{ZZZZ-25-FMMNN}
Shijun {Zhang}, Hongkai {Zhao}, Yimin {Zhong}, and Haomin {Zhou}.
\newblock {F}ourier multi-component and multi-layer neural networks: Unlocking
  high-frequency potential.
\newblock \emph{Neural Networks}, 204:\penalty0 109268, 2026.
\newblock ISSN 0893-6080.
\newblock \doi{10.1016/j.neunet.2026.109268}.

\bibitem[{Zhang} et~al.(to appear){Zhang}, {Shen}, and
  {Xu}]{shijun:mgdl:relu:decay}
Shijun {Zhang}, Zuowei {Shen}, and Yuesheng {Xu}.
\newblock Multigrade neural network approximation.
\newblock \emph{Communications of the American Mathematical Society}, to
  appear.
\newblock URL: \url{https://arxiv.org/abs/2601.16884}.

\bibitem[Zhou(2020{\natexlab{a}})]{ZHOU2019}
Ding-Xuan Zhou.
\newblock Universality of deep convolutional neural networks.
\newblock \emph{Applied and Computational Harmonic Analysis}, 48\penalty0
  (2):\penalty0 787--794, 2020{\natexlab{a}}.
\newblock ISSN 1063-5203.
\newblock \doi{10.1016/j.acha.2019.06.004}.

\bibitem[Zhou(2020{\natexlab{b}})]{Zhou2020Downsampling}
Ding-Xuan Zhou.
\newblock Theory of deep convolutional neural networks: Downsampling.
\newblock \emph{Neural Networks}, 124:\penalty0 319--327, 2020{\natexlab{b}}.
\newblock \doi{10.1016/j.neunet.2020.01.018}.

\end{thebibliography}
\end{document}